\documentclass[nohyperref]{article}

\usepackage{microtype}
\usepackage{graphicx}
\usepackage{subfigure}
\usepackage{booktabs} % for professional tables

\usepackage{hyperref}

\usepackage[noend]{algorithmic}

\usepackage[accepted]{icml2023}

\usepackage{amsmath}
\usepackage{amssymb}
\usepackage{mathtools}
\usepackage{amsthm}

\usepackage[capitalize,noabbrev]{cleveref}
\usepackage{listings}
\theoremstyle{plain}
\newtheorem{theorem}{Theorem}[section]

\newtheorem{lemma}[theorem]{Lemma}
\newtheorem{corollary}[theorem]{Corollary}
\theoremstyle{definition}

\theoremstyle{remark}

\usepackage[textsize=tiny]{todonotes}

\icmltitlerunning{FedHQL: Federated Heterogeneous Q-Learning}

\begin{document}

\twocolumn[
\icmltitle{FedHQL: Federated Heterogeneous Q-Learning}

% It is OKAY to include author information, even for blind
% submissions: the style file will automatically remove it for you
% unless you've provided the [accepted] option to the icml2023
% package.

% List of affiliations: The first argument should be a (short)
% identifier you will use later to specify author affiliations
% Academic affiliations should list Department, University, City, Region, Country
% Industry affiliations should list Company, City, Region, Country

% You can specify symbols, otherwise they are numbered in order.
% Ideally, you should not use this facility. Affiliations will be numbered
% in order of appearance and this is the preferred way.
\icmlsetsymbol{equal}{*}

\begin{icmlauthorlist}
\icmlauthor{Flint Xiaofeng Fan}{NUS1,ETH,Astar}
\icmlauthor{Yining Ma}{NUS1}
\icmlauthor{Zhongxiang Dai}{NUS1}
\icmlauthor{Cheston Tan}{Astar}
\icmlauthor{Bryan Kian Hsiang Low}{NUS1}
\icmlauthor{Roger Wattenhofer}{ETH}

%\icmlauthor{}{sch}
%\icmlauthor{}{sch}
\end{icmlauthorlist}

\icmlaffiliation{NUS1}{National University of Singapore, Singapore}
% \icmlaffiliation{NUS2}{Department of ISEM, National University of Singapore, Singapore}
\icmlaffiliation{ETH}{ETH
Zurich, Switzerland}
\icmlaffiliation{Astar}{Institute for Infocomm Research, Singapore}
% \icmlaffiliation{sch}{School of ZZZ, Institute of WWW, Location, Country}

\icmlcorrespondingauthor{Flint Xiaofeng Fan}{fxf@u.nus.edu}
% \icmlcorrespondingauthor{Firstname2 Lastname2}{first2.last2@www.uk}

% You may provide any keywords that you
% find helpful for describing your paper; these are used to populate
% the "keywords" metadata in the PDF but will not be shown in the document
\icmlkeywords{Federated Reinforcement Learning, Federated Q-learning, Federated Learning}

\vskip 0.3in
]

% this must go after the closing bracket ] following \twocolumn[ ...

% This command actually creates the footnote in the first column
% listing the affiliations and the copyright notice.
% The command takes one argument, which is text to display at the start of the footnote.
% The \icmlEqualContribution command is standard text for equal contribution.
% Remove it (just {}) if you do not need this facility.

\printAffiliationsAndNotice{}  % leave blank if no need to mention equal contribution
% \printAffiliationsAndNotice{\icmlEqualContribution} % otherwise use the standard text.

\begin{abstract}
Federated Reinforcement Learning (FedRL) encourages distributed agents to learn collectively from each other’s experience to improve their performance without exchanging their raw trajectories. The existing work on FedRL assumes that all participating agents are homogeneous, which requires all agents to share 
the same policy parameterization (e.g., network architectures and training configurations). 
However, in real-world applications, agents are often in disagreement about the architecture and the parameters, possibly also because of disparate computational budgets.
Because homogeneity is not given in practice, we introduce the problem setting of Federated Reinforcement Learning with Heterogeneous And bLack-box agEnts (FedRL-HALE).
We present the unique challenges this new setting poses and propose the Federated Heterogeneous Q-Learning (FedHQL) algorithm that principally addresses these challenges.
We empirically demonstrate the efficacy of FedHQL in boosting the sample efficiency of heterogeneous agents with distinct policy parameterization
using standard RL tasks.
\end{abstract}

\section{Introduction}\label{sec:introduction}
Leveraging on the growing literature of \emph{federated learning} (FL) \citep[etc.]{mcmahan2017communication-FL-0,konevcny2016federated-FL-1,kairouz2021advances-FL-2}, \emph{federated reinforcement learning} (FedRL) 
% \citep[etc.]{zhuo2019federated,FedRL-2,FedRL-1,liang2019federated-FedRL-Car-YQ,yu2020deep-FedRL-3,fan2021fault-FRL,wang2020federated-FedRL-5,jin2022federated-FRL-aistats,khodadadian2022federated-FRL-icml} 
\cite{zhuo2019federated}
has recently become an emerging approach
to enable collective intelligence \citep{yahya2017collective-distributed-RL} 
in sequential decision-making problems
with a large number of distributed reinforcement learning (RL) agents. 
In contrast to the conventional setting of \emph{distributed} RL in which the raw trajectories of agent-environment interactions may be collected by a central server \citep[etc.]{nair2015massively-DRL-1,mnih2016asynchronous-DRL-2,espeholt2018impala-DRL-3,horgan2018-distributed-DRL-4,chen2022byzantine-DRL}, FedRL, as a specialization of distributed RL, places its emphasis 
on the accessibility of the raw trajectories of agents in real-world domains such as e-commerce and healthcare.
% on the privacy aspects of agents -- transmitting raw trajectories of agent-environment interactions is prohibitive in FedRL.
% existing applications of FedRL
% There are many practical scenarios where RL trajectories is infeasible and hence FedRL applications. 
For example, in the healthcare domain, sharing trajectories corresponding to the medical records of patients is prohibitive \cite{rieke2020future-nature}.
Thus,
applications like the clinical decision support system proposed by
\citet{xue2021resource-FRL-clinical} need to utilize FedRL to extract knowledge from electronic medical records from distributed sources where the raw records are inaccessible directly. 
% In such an application, transmitting the trajectories of medical records is prohibitive \cite{rieke2020future-nature}.
% 
FedRL may also be applied to practical applications where transmitting RL trajectories is infeasible due to limited hardware capacities and transmission bandwidth.
% For instance, e-commerce companies can develop FedRL applications to provide personalized services to a large number of customers without accessing the customers' data \cite{FedRL-2}; 
% when hardware capacity and transmission bandwidth are limited, 
For instance, \citet{FedRL-1} develops a FedRL server in the cloud for a group of robots
% framework
to learn to navigate cooperatively through the world 
without transmitting their 
% full
observation trajectories to the server to save bandwidth.
Besides those promising practical applications, many endeavors have been made to study the theoretical aspects of FedRL. For example,
FedRL has recently been theoretically proved to be effective in improving the sample efficiency of RL agents proportionally with respect to the number of agents with performance guarantees \citep{fan2021fault-FRL}; the work of \citet{jin2022federated-FRL-aistats} has proved the convergence of FedRL algorithms with agents operating in distributed environments with different state-transition dynamics;
\citet{khodadadian2022federated-FRL-icml} has further proved that FedRL can provide a linear 
% speed 
convergence speedup with respect to the number of agents under Markovian noise.
% DP
% \cite{zhao2023federated-DP} 

Despite their promising theoretical results \citep{fan2021fault-FRL,jin2022federated-FRL-aistats,khodadadian2022federated-FRL-icml} and practical applications \citep{FedRL-2,FedRL-1,yu2020deep-FedRL-3,wang2020federated-FedRL-5,FedRL-building,liang2019federated-FedRL-Car-YQ}, 
current FedRL algorithms explicitly assume that all participants are \emph{homogeneous}, 
% which can be a significant limitation in practice.
% The assumption on the homogeneity of agents 
which requires all agents to share the same policy parameterization (e.g., the architecture of the policy neural network, including the number of layers, the activation function, etc.) and the same training configurations for the policy (e.g., the learning rate).
% .
%
Such an assumption can be a significant limitation in
real-world applications where 
agents are often \emph{heterogeneous}, due to various factors such as different computational budgets, different assessments of the difficulty of the task, etc.
% 
% However, different agents may favor distinct policy network architectures and policy training configurations due to various factors such as different computational budgets, different assessments of the difficulty of the task, etc.
% 
% These issues make it critical for FedRL to cater to \emph{heterogeneous} agents.
% % Those issues make it critical for FedRL to cater for \emph{heterogeneous} agents.
% To fill this gap, 
To solve this problem,
we propose a federated version of Q-learning \citep{watkins1989learning-Q-learning} that works with $N$ distributed agents that are heterogeneous in terms of a number of factors including the policy parameterization (e.g., the policy network architecture), the policy training configurations (e.g., the learning rate) and the exploration strategy to manage the trade-off between exploration and exploitation.
% To fill this gap, we propose a federated version of Q-learning that works with $N$ distributed agents who are heterogeneous in terms of utilizing different architectures of networks. 
% The agents may also choose different training configurations such as training rates and different exploration strategies to trade-off between exploration and exploitation.
% There is a rich literature on different RL algorithms catering different aspects of the decision-making process such as DQN \citep{mnih2013playing-DQN}, A2C \citep{mnih2016asynchronous-A2C}, and PPO \citep{schulman2017proximal-PPO}. It is thus unnatural to require all agents to follow the same RL algorithm since different agents with different domain knowledge may have different assessment of the underlying task. 
% continue

Another important issue faced by existing FedRL algorithms results from the assumption that the server has full knowledge about the policy-related details of agents, which is a privacy concern to the agents. 
% Another issue faced FedRL comes from the assumption that the server has full knowledge of agents, which poses privacy concerns to the agents. 
In the aforementioned previous works on FedRL, the server has access to information such as the architectures of the policy networks, the details on how the agents train and update their policy networks, and the exploration strategy of the agents to trade off exploration and exploitation.
Such information may reveal critical information about the agents.
% In the aforementioned works of FedRL, the server has access to information such as the detailed architectures of the policy networks, the information on how the agents train and update their networks, and the exploration strategy of an agent to trade-off between exploration and exploitation, which may reveal critical information about the agents.
For example, consider a FedRL application in the financial market where the agents are different organizations. Knowing that one organization utilizes a computationally expensive transformer network  \cite{chen2021decision-transformer} may 
imply the organization's relations with customers and excellent financial standing.
% reveal the organization's 
% and vice versa.
% For example, consider a FedRL application in the financial markets where agents are different organizations. Knowing that one organization utilizes a computationally expensive network model such as Transformers \cite{chen2021decision-transformer} may imply the organization's excellent economic and financial standing and vice versa.
To address this issue, our proposed federated Q-learning algorithm treats each agent as a black-box expert whose 
% policy-related details, 
% such as its policy network architecture, policy training configurations, and exploration strategy, 
policy network parameterization, training configuration and exploration strategy
are hidden from any other party including the server.
% Towards this end, our proposed federated Q-learning algorithm treats each agent as a black-box expert whose internal architectures of network models, detailed configuration of model training, and exploration strategies are not shared with any other party.
% given a state,, the server or any party can infer what action decision an agent will make...
% , without which the privacy of Rl agents is hard to be guaranteed...
%
% our proposed algorithm ... only the distribution of action values is answered by agents which is provided to the server for aggregating the knowledge...

In this work, we propose the problem setting of \emph{\textbf{Fed}erated \textbf{R}einforcement \textbf{L}earning with \textbf{H}eterogeneous \textbf{A}nd {b}\textbf{L}ack-box {a}g\textbf{E}nts} (FedRL-HALE).
In the setting of FedRL-HALE, which is illustrated in Fig.~\ref{fig:fedRL-diagram}, every agent is free to choose its own preferred policy parameterization, training configurations, and exploration strategy.
Furthermore, the agents in the setting of FedRL-HALE are black-box experts whose policy-related information, such as those mentioned above, is not shared with any other party including the server. 
% In this work, we propose to study the problem setting of Federated Reinforcement Learning with heterogeneous black-box agents (FedRL-HALE), as illustrated in Fig.~\ref{fig:fedRL-diagram}, in which each agent is free to choose its preferred exploration strategy to trade-off between exploration and exploitation and agents may utilize different network architectures and training configurations.
% Furthermore, the agents in FedRL-HALE remain as black-box experts whose information mentioned above is not shared with any other party or accessible by the server. 
% the central server does not have access to the information about its choice of RL algorithm and network parameterization details of each agents. \emph{This line I think is very important. Do not let them think we are still sharing Q table thus blahblah..} The server is only allowed to query, from the \emph{black-box} agents, the estimated action-value (TODO: hyperlink. to be defined in Sec~()) conditioned on the states generated by the server. 
To achieve collaborative learning of heterogeneous and black-box agents in the setting of FedRL-HALE, we introduce the \textbf{Fed}erated \textbf{H}eterogeneous \textbf{Q}-\textbf{L}earning (FedHQL) algorithm.
To allow the server to aggregate information from different agents and use it to improve their learning, we propose (\emph{a}) the Federated Upper Confidence Bound (FedUCB) algorithm, which is deployed by the server as a subroutine of the FedHQL algorithm, and (\emph{b}) a federated version of Temporal Difference (FedTD) which is calculated by the server and then broadcast to all agents to improve their policies.
% To achieve this, we first propose a Federated Q-Learning system accounting for $N$ heterogeneous black-box agents and develop a Federated Upper Confidence Bound algorithm (FedUCB) that provides a principled way of aggregating the knowledge of agents by a central server that periodically queries each agent. The aggregated knowledge using FedUCB is then updated using a federated version of Temporal Difference (FedTD) which is then broadcast back to each agent to help regularize their knowledge.

% In particular, we have made the following contributions:
Specifically, we make the following contributions:

\textbf{1.} We 
% present 
introduce
the problem formulation of FedRL-HALE and discuss its technical challenges (Sec.~\ref{sec:problem-setting}).

\textbf{2.} We develop the FedUCB algorithm, which provides a principled strategy for \emph{inter}-agent exploration to balance the trade-off between exploration and exploitation in the setting of Fed-HBA (Sec.~\ref{sec:fed-ucb-q}).
% \textbf{2.} We develop a FedUCB algorithm to provide a principled \emph{inter}-agent exploration strategy to balance the trade-off between exploration and exploitation in the federated learning setup (Sec.~\ref{sec:fed-ucb-q}).

\textbf{3.} We propose the FedHQL algorithm, which deploys FedUCB as a subroutine and is the first FedRL algorithm catered to heterogeneous and black-box agents (Sec.~\ref{subsec:algorithm}).
% \textbf{3.} We propose \textbf{Fed}erated \textbf{H}eterogeneous \textbf{Q}-\textbf{L}earning (FedHQL), the first FedRL algorithm that cater for heterogeneous black-box agents (Sec.~\ref{subsec:algorithm}).

\textbf{4.} We conduct extensive empirical evaluations to demonstrate the efficacy of FedHQL in improving the sample efficiency of agents using standard RL tasks (Sec.~\ref{sec:experiments}). 
% We discuss the limitations and potential future works in Sec.~\ref{sec:discussion}.
% \textbf{4.} We conduct an extensive empirical study that demonstrates the efficacy of FedHQL in improving the sample efficiency of agents on popular RL benchmark (Sec.~\ref{sec:experiments}). We discuss the limitations and potential future works in Sec.~\ref{sec:discussion}.
% The main contributions of this paper include the FedRL with Heterogeneous Black-box agents problem setting itself and \emph{Federated Q-Learning} (Fed-Q-L), a novel FedRL framework that addresses the following challenges: (a) performing federation without access to the policy architectures and RL control algorithms of the black-box agents; (b) providing a UCB-style exploration strategy to each participating agent.
% We will formulate the problem and discuss the technical challenges of this setting in Sec.~\ref{sec:problem-setting}.
% % (c) avoiding single point of failure that may break the system.
% % After 
% % surveying related work in Sec.~\ref{sec:related-work} and 
% % providing background in Sec.~\ref{sec:background}, we formulate the problem setting of in Sec.~\ref{sec:problem-setting} and 
% We then present the details of Fed-Q-BA in Sec.~\ref{sec:algorithm}. The effectiveness of Fed-Q-BA is empirically verified through carefully designed experiments in Sec.~\ref{sec:experiments}. Limitations of our proposed framework and potential future works are discussed in Sec.~\ref{sec:discussion}  
% To the best of our knowledge,,, the....FedRL with blackbox agent is not avaiable in the literature...

\begin{figure}[t]
    \centering
    % \noindent
    % \makebox[\textwidth]{\includegraphics[width=5.5in, height=1.2in] {./plots/EXP1.pdf}}
    \includegraphics[height=1.8in] {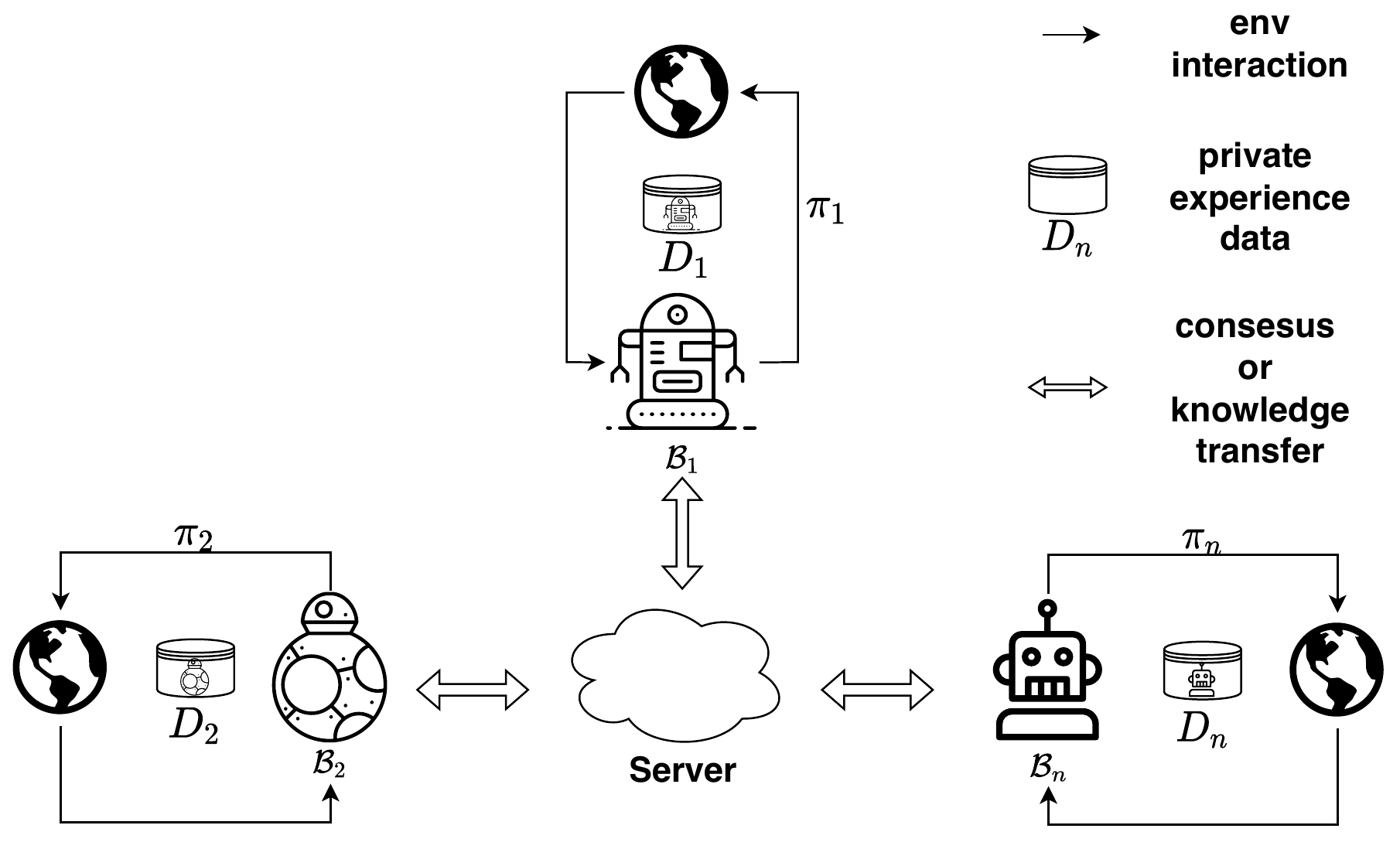}
    % \vspace{-8pt}
    \caption{Graphical illustration of the FedRL-HALE problem setup (Sec.~\ref{subsec:problem-formulation}). $\mathcal{B}_n$ represents a black-box agent whose policy-related information 
    % as
    and
    private experience data $D_n$ are not accessible to any other party.
    A central server 
    % is to perform
    performs
    certain knowledge 
    % fusion
    aggregation
    process in order to improve the performance of participating agents without accessing the aforementioned private information of the agents. 
    % with $K=10$ agents among which 3 are FedPG attackers.
    }
    \label{fig:fedRL-diagram}
\end{figure}
%%%%%%%%%%%%%%%%%%%%%%%%%%%%%%%%%%%%%%%%%%%%%%%%%%%%%%%%%%%%%%%%%%%%%%%%

\section{Preliminaries}\label{sec:background}
\subsection{Markov Decision Processes}\label{def:MDP}
We model Reinforcement Learning (RL) as a discrete-time episodic Markov Decision Process (MDP) \citep{sutton2018reinforcement-RL}: $M \triangleq \{\mathcal{S}, \mathcal{A}, \mathcal{P}, \mathcal{R}, \gamma, \rho, H\}$ where $\mathcal{S}$ and $\mathcal{A}$ are the state space and action space of the RL task, respectively. 
% The 
$\mathcal{P}(s^\prime | s, a)$ defines the transition probability of the environment,
% from state $s$ to $s^\prime$
% % that the agent transits to state $s^\prime$ from $s$ 
% after taking action $a$,
$\mathcal{R}(s,a): \mathcal{S} \times \mathcal{A} \mapsto[0, R] $ denotes the reward function,
% for state-action pair $(s,a)$ and some constant $R > 0$,
$\gamma\!\in\!(0,1)$ is the discount factor,
$\rho$ represents the initial state distribution, and $H$ denotes the task horizon of an episode.
An agent behaves according to its policy $\pi$, where $\pi(a|s)$ defines the probability of the agent choosing action $a$ at state $s$. 
$\tau \triangleq \{s_0, a_0, s_1, a_1, ..., s_{H-1}, a_{H-1}|s_0\!\sim\!\rho\}$ is a trajectory of 
agent-environment interactions.
% state-action pairs
% (a.k.a. agent-environment interactions) where $s_0\!\sim\!\rho$. 
%  $\mathcal{R}(\tau) \triangleq\sum_{t=0}^{H-1} \gamma^{t} \mathcal{R}\left(s_{t}, a_{t}\right)$ gives the cumulative discounted reward for a trajectory $\tau$.
%
The return function $R_t \triangleq\sum_{i=0}^{H  } \gamma^{i} \mathcal{R}\left(s_{t+i}, a_{t+i}\right)$ gives the cumulative discounted reward from timestep $t$ in $\tau$ that an RL algorithm aims to maximize.

\subsection{Q-learning}\label{subsec:q-learning}
One of the most important breakthroughs in RL was the development of \emph{Q-learning} \cite{watkins1989learning-Q-learning,sutton2018reinforcement-RL} which 
% substantially
launched the field of deep RL by using neural networks as function approximators~\citep{mnih2013playing-DQN}. The core of this model-free RL algorithm is to iteratively apply the \emph{Bellman equation} to update the \emph{action-value function} $Q$:
\begin{align}\label{eq:q-update-bg}
    Q(s_t, a_t) \leftarrow Q&(s_t, a_t) + \alpha \Big[\mathcal{R}(s_{t+1}, a_{t+1})  \notag \\
    &+ \gamma \max_a Q(s_{t+1}, a) - Q(s_t, a_t) \Big]
\end{align}
where $Q(s,a)\!:\!\mathcal{S} \times \mathcal{A} \mapsto \mathbb{R}$ represents the value of an action $a$ at state $s$, and $\alpha$ is the learning rate. 
Then a decision policy $\pi(a_t|s_t)$ can be obtained via \emph{exploiting} the updated $Q$:
\begin{equation}
\label{eq:action:selection:greedy}
    a_t = \arg \max_a Q(s_t, a) .
\end{equation}
The \emph{optimal action} at state $s_t$ is defined as $a^*_t = \arg \max_a Q^*(s_t, a)$
% \begin{equation}
%     a^*_t = \arg \max_a Q^*(s_t, a) .
% \end{equation}
where $Q^*(s_t, a)$ is the \emph{optimal Q-function} which gives the expected return for starting in state $s$, taking action $a$, and following the policy thereafter.

One key characteristic of Q-learning is that the optimization step is always performed \emph{off-policy} meaning that each update of Eq.~\eqref{eq:q-update-bg} can use any data samples collected at any point, regardless of how the agent deals with the exploration-exploitation dilemma (Sec.~\ref{subsec:exploration-problem}) 
to collect the samples.
% when the data sample was collected.
%
\subsection{The Exploration-Exploitation Dilemma}\label{subsec:exploration-problem}
During the learning process, when an agent makes a decision that seems optimal according to Eq.~\eqref{eq:action:selection:greedy}, it is \emph{exploiting} its current knowledge which is assumed to be reliable; the agent may also selects an action that seems sub-optimal for now to potentially \emph{explore} unseen states, making the assumption that its current knowledge could be incomplete and inaccurate. 
This 
% situation 
is known as the \emph{exploration-exploitation dilemma}, which is omnipresent in many sequential decision-making problems.
% resulting in one of the causes of instability in Q-learning \cite{szepesvari2010algorithmsRL-review}.
%
To balance the trade-off between exploration and exploitation, a variety of exploration strategies has been developed,
such as $\epsilon$-greedy and Boltzmann exploration~\cite{cesa2017boltzmann}.
% , are often sample-inefficient, resulting in that Q-learning fails to converge to the optimal policy \citep{watkins1989learning-Q-learning}. 
Among them, the $\epsilon$-greedy algorithm is widely used in Q-learning as its exploration strategy:
\begin{align*}
    a_t = \begin{cases}\arg \max _{a} {Q}(s_t, a) & \text { with probability } 1-\epsilon, \\ \text {random action } a & \text { with probability } \epsilon\end{cases}
\end{align*}
where the coefficient $\epsilon$ is a hyper-parameter that controls the degree of exploration.

One optimism-based exploration strategy is the Upper Confidence Bound (UCB) algorithms \citep{lattimore2020bandit-book},
% One optimism-based exploration strategy that often leads to sublinear regret in the bandits problem is the Upper Confidence Bound (UCB) algorithm \citep{lattimore2020bandit-book}.
which is established based on the principle of optimism in the face of uncertainty.
% handle xxx..xxx tradeoff in the principled way
% The UCB algorithms adjust the exploration and exploitation balance as more knowledge of the environment is gathered. 
The action chosen by UCB at time step $t$ is given by:
\begin{equation}\label{eq:action-selection}
    a_t = \arg \max_a \Big[Q(s_t, a) + c\sqrt{\frac{\log t}{N_t(a)}} \Big]
\end{equation}
where $Q(s_t,a)$ is the current estimated value of action $a$ at time step $t$; $N_t(a)$ is the number of times action $a$ has been selected prior to time $t$; $c$ is a confidence constant that controls the level of exploration.
The UCB algorithms construct a high-probability upper bound on $Q(s_t, a)$ and automatically balance exploration and exploitation as more knowledge of the environment is gathered.
% or put it in next section. give a brief literature review on UCB exploration
%
% \subsection{Federated Reinforcement Learning}
% (May not need it since we have talked about it in the intro.)It is well established that RL algorithms suffer from poor sample efficiency in real-world applications \cite{dulac2019challengesRealWorld-RL,levine2020offlineRL}, which has motivated the development of federated RL (FedRL) \cite{zhuo2019federatedRLYangQ-YQ}.
% lets do TD here instead
% \subsection{Temporal Difference }

\section{FedRL-HALE}\label{sec:problem-setting} 
% \section{FedRL with Heterogeneous Black-box agents (FedRL-HALE)}\label{sec:problem-setting} 
% \subsection{Notations}
% We use 
% whenever ... $\bar{\{{\cdot}\}}$ denotes \emph{federated} version of ${\{{\cdot}\}}$; ${\{{\cdot}\}}^k$ denotes the property ${\{{\cdot}\}}$ of agents $k$;... for clarity, we use environment steps to refer to number of transitions with the environment...
\subsection{Problem formulation}\label{subsec:problem-formulation}
Consider the task of \emph{federatively} solving a sequential decision-making problem represented by the MDP $M \triangleq \{\mathcal{S}, \mathcal{A}, \mathcal{P}, \mathcal{R}, \gamma, \rho, T\}$ defined in Sec.~\ref{def:MDP}. 
Let the set $\mathcal{B} \triangleq\{Q_n(a |\ell_n(s;{D}_n, \omega_n ))\}_{n=1}^{N}$ denote a group of $N$ distributed heterogeneous and black-box agents. 
Following the same setup as the existing literature of FedRL (Sec.~\ref{sec:introduction}), each agent $\mathcal{B}_n$
independently operates in a separate copy of the underlying MDP $M$ following its policy $\pi_n$, and generates its private experience data $D_n \triangleq \{(s, a, s^{\prime}, r)_i\}_{i=1}^{|D_n|}$. 
% The action 
% 
Each action valuator $Q_n(a |\ell_n(s;{D}_n, \omega_n ))$ consists of a non-linear function $\ell_n(s;{D}_n, \omega_n)$ which predicts the value of action $a$ given a state $s$.
The non-linear function $\ell_n(s;{D}_n, \omega_n)$ is parameterized by a set of parameters $\omega_n$ and learned using the private experience data ${D}_n$.
% Each action valuator $Q_n(a |\ell_n(s;{D}_n, \omega_n ))$ is parameterized by (a) a non-linear function $\ell_n(s;{D}_n, \omega_n)$ that acts as learned knowledge of the query state $s$ given private experience $D_n$; and (b) a set of characterizing parameters $\omega_n$ % of $\ell_n$
% accounting for its predictive uncertainty of $\ell_n$.

A prominent example of such a non-linear function $\ell_n$ is a neural network, in which $\omega_n$ represents its weights.
Due to the \emph{heterogeneity} among agents, different agents may choose different neural network architectures and employ different optimization methods to train their networks. 
% One such parameterization of $\ell_n$ and $\omega$ is the use of neural networks. To account for \emph{heterogeneity}, agents may choose different architectures of neural networks and employ different optimization methods to train the networks. 
\emph{For example, agent $\mathcal{B}_1$ favours a Transformer-like network architecture and trains it with the Adam Optimizer \cite{chen2021decision-transformer}; agent $\mathcal{B}_2$ can choose to use a 3-layer multi-layer perceptron (MLP) as the network and train it using conventional stochastic gradient ascent; meanwhile, agent $\mathcal{B}_3$ who also adopts an MLP as the network may utilize an SVRG-type \cite{johnson2013svrg1,papini2018stochastic-svrpg} optimizer for training.} 
% \emph{For example, agent $\mathcal{B}_1$ favours a Transformer-like network parameterization and training \cite{chen2021decision-transformer} while agent $\mathcal{B}_2$ may choose to use a 3-layer-perceptron as the network and train it using conventional stochastic gradient ascent, which is then different from agent $\mathcal{B}_3$ who utilizes the SVRG-like \cite{johnson2013svrg1,papini2018stochastic-svrpg} optimizer for training the network.} 
{As aforementioned, the neural network $\ell_n(s; {D}_n,\omega_n)$ is trained 
% independently 
on a private set of experience data $D_n
% \triangleq \{\tau_i\}_{i=1}^{|D_n|}$ where $\tau \triangleq \{s_0, a_0, s_1, a_1, ..., s_{H-1}, a_{H-1}\}
$ for agent $\mathcal{B}_n$. And we note that in any given state $s$, the neural network-based action valuator $Q_n(a |\ell_n(s;{D}_n, \omega_n ))$ is able to predict the value of any action $a$.}
To simplify exposition, in the remainder of the paper, we use $Q_n(s,a)$ to denote the action valuator $Q_n(a |\ell_n(s;{D}_n, \omega_n ))$.
% Both $\ell_n(s, {D}_n)$ and $\omega_n$ are trained independently on a private set of experience data $D_n
% % \triangleq \{\tau_i\}_{i=1}^{|D_n|}$ where $\tau \triangleq \{s_0, a_0, s_1, a_1, ..., s_{H-1}, a_{H-1}\}
% $.
%
% Thus, given the action value distribution, the action to execute at input state $s$ can be determined greedily by local agent $\mathcal{B}_n$ with the highest expected value $a_n(s) = \arg\max_a Q_n(a |\ell_n(s;{D}_n, \omega_n ))$.  

% Each agents $k$ constructs its unique policy $\pi^k$ based on its own assessment of the task and its computational budget. 

% agents perform independent learning via interacting with the MDP and 
Similar to the settings in the existing FedRL works \cite{fan2021fault-FRL,jin2022federated-FRL-aistats,khodadadian2022federated-FRL-icml}, a reliable central server is available to coordinate the FedRL process and is allowed to interact with another separate copy of the underlying MDP \cite{fan2021fault-FRL}. 
% \textcolor{red}{should we change the $q^k(a|x)=\mathbb{E}_{y \sim \mathcal{P}}[R(x,a) + \gamma \underset{{b \sim \pi^k}}{\max} \  q^k (b|y)]$ to sth more general such that we say our formulation is general in the sense future works may consider replace the q(a|x) with sth else... so we just say this is some limited information...anyway, talk about the q(a|x) here.....and $x \in \mathcal{X}$} 
At any point during the FedRL training process, the server may query all \emph{black-box} agents with a selected state $\bar{s}$, and then collect from every agent $\mathcal{B}_n$ the action values 
% $Q_n(a |\ell_n(\bar{s};{D}_n, \omega_n ))$ 
$Q_n(\bar{s}, a)$
for all actions $a$'s.\footnote{{The query could be a batch of states sent together in one communication round in practice. Here we omit the batch size and illustrate our FedDRL-HBA setting (Sec.~\ref{sec:problem-setting}) and our FedHQL algorithm (Sec.~\ref{sec:FedHQL}) with the case of a single state for brevity.}}
% At any point during the FedRL training process, the server may query all \emph{black-box} agents with a selected state $\bar{s}$ and then collect the action value distributions $Q_n(a|\bar{s})$ from all agents.
The server then needs to design a mechanism for aggregating the collected information and broadcasting it back to all agents to aid their individual training.
% \textcolor{blue}{To put a stronger emphasis on the privacy of participating agents, }  agents $k$'s interaction history $\mathcal{D}^k$ and information regarding its the policy $\pi^k$, such as the RL algorithm on which the policy is trained, the action selection method, and the parameterization details of the neural network, etc., are not accessible to any other parties including the server. 
% FedRL algorithm
% Of note, the agents are \emph{black-box}
% Of note, the non-linear function $\ell_n$, the weights $\omega_n$, and the training methods and training details, together with the local experience data $D_n$, and not revealed to any other party including the central server.
Of note, the information regarding the non-linear function $\ell_n$ (including its architecture, weights $\omega_n$, training methods, and other training details), as well as the local experience data $D_n$, is not revealed to any other party, including the central server.

% objective
Let $R_n$ and $\overline{R}_{{n}}$ denote the performances achievable by agent $\mathcal{B}_n$ through independent learning and FedRL-HALE, respectively.
Let $|D_n|$ and $|\overline{D}_n|$ denote, respectively, the number of agent-environment interactions
% of agent $\mathcal{B}_n$ 
required to reach $R_n$ through independent learning and to reach $\overline{R}_n$ through FedRL-HALE.
$|\overline{D}_s|$
denotes the total number of interactions incurred at the server. We define two objectives in the setting of FedRL-HALE:
% We define two objective of the FedRL-HALE problem:
% (1) from the server's perspective, for the server to maximize the expected return through collective intelligence under the federated learning setup, i.e.,
% \begin{align*}
%     \max \mathbb{E}\Big[\sum_{t=1}^T \gamma^t R(x_t, \bar{a}_t)\Big| x_{0} \sim \rho_0, \bar{a}_t \leftarrow \text{FedRL}(\mathcal{B}, x_t),  x_{t+1} \sim \mathcal{P}(x_{t+1} |x_t, \bar{a}_t)\Big]
% \end{align*}

(1) \emph{from the perspective of improving the overall system welfare}: 
% let $R_{\text{mean}}$ and $\overline{R}_{\text{mean}}$ denote the performance achievable averaged over all agents in independent learning and FedRL-HALE correspondingly.
% Let $|\overline{D}_s|$ denote the total number of interactions incurred at the server.
We aim to improve the average performance of all agents, given a fixed budget on the number of interactions with the whole system:
\begin{align}\label{eq:objective-1}
     \underbrace{\frac{\sum_n \overline{R}_n}{N}}_{\text{FedRL-HALE}} \geq \underbrace{\frac{\sum_n {R}_n}{N}}_{\text{self-learning}}
     s.t.\ \underbrace{|\overline{{D}}_s| + \sum_n |\overline{D}_n| }_{\text{FedRL-HALE}}  \leq \underbrace{\sum_n  |{D}_n|.}_{\text{self-learning}}
\end{align}

(2) \emph{from the perspective of each participating agent $\mathcal{B}_n$} who aims to improve its performance with \emph{less} agent-environment interactions compared to that of independent learning: 
% Let $R_n$ and $\overline{R}_{{n}}$ denote the performances achievable by agent $\mathcal{B}_n$ through independent learning and FedRL-HALE, respectively.
% given a performance threshold, $|{D}_n| < |\overline{D_n}|$ where 
% Let $D_n$ and $|\overline{D}_n|$ denote, respectively, the number of agent-environment interactions
% % of agent $\mathcal{B}_n$ 
% required to reach $R_n$ through independent learning and to reach $\overline{R}_n$ through FedRL-HALE.
We aim to achieve:
\begin{align}\label{eq:objective-2}
     \underbrace{|\overline{D}_n|}_{\text{FedRL-HALE}}   \leq \underbrace{|{D_n}|}_{\text{self-learning}} s.t.\ \underbrace{\overline{R}_n}_{\text{FedRL-HALE}} \geq \underbrace{R_n.}_{\text{self-learning}}
\end{align}

% to provide the incentive for agent $k$ to participate in the federation such that - the return function $R_t^{(k)}|_{k \in K}$ can be better maximized compared to that of independent learning; or the return function $R_t^{(k)}|_{k \in K}$ can be maximized as good as that of independent learning but with less samples (smaller $N_k$) required.

% One commonly used strategy for $\pi_k$ is defined in Equation~\eqref{eq:action-selection} which chooses greedy action $a$ according to $q_k$. The action-value function $q_k$, parameterized by $\theta_k$, calculates the quality of an action $a$ at state $s$ and is trained independently on its private experience data $\mathcal{D}_k \triangleq \{\tau_i\}_{i=1}^{N_k}$ where the size $N_k$ depends on the capacity of agent $k$.
% %
% % its private replay buffer $\mathbf{D}_k \triangleq \{(s_{t+i}, a_{t+i}, r_{t+i}, s_{t+1+i})\}_{i=1}^{N_k}$ where the size $N_k$ depends on the capacity of agent $k$. The use of replay buffer has been well studied and proven to improve the sample efficiency of model-free RL algorithms \citep{fedus2020revisiting-replay-buffer}.
% %
% Each agent chooses their own RL algorithm to optimize the parameter $\theta_k$.
% %
% % setting
% Similar to the common setups of FL\cite{mcmahan2017communication-FL-0}, a trustworthy server is governing and coordinating the system.
% The agents are considered to be self-interested and their policies $\pi_k$ and parameters $\theta_k$, as well as their private experience data $\mathcal{D}_k$ are not accessible to any other party. 
\subsection{Technical Challenges}\label{sec:technical-challenges}
Here we discuss unique technical challenges that arise in the above-proposed setting of FedRL-HALE.
% \begin{itemize}
%     \item Black-box optimization. 
% \end{itemize}

\textbf{Black-box Optimization.} Since the details of policy network architecture, training configuration, and exploration strategy of every agent are not revealed to any other party including the server,
it is challenging to construct and optimize the objective function of FedRL-HALE,
resulting in a distributed black-box optimization problem \cite{bajaj2021-blackbox-review}.
It is thus infeasible to solve the FedRL-HALE problem with conventional gradient-based methods \cite{fan2021fault-FRL} 
or parameter-based methods \cite{khodadadian2022federated-FRL-icml} for knowledge aggregation.
{Meanwhile,} the nature of black-box optimization renders conventional federated aggregation methods from the federated learning literature, such as federated averaging (FedAvg) \cite{mcmahan2017communication-FL-0}, inapplicable in FedRL-HALE. 
% The black-box optimization nature renders conventional federated aggregation methods from the Federated Learning literature such as FedAvg \cite{mcmahan2017communication-FL-0} unlearnable in FedRL-HALE. 

To this end, 
we propose a federated version of Q-learning (Sec.~\ref{subsec:q-learning}) that aggregates 
% over
the action value estimations of different agents $Q_n$'s on the states queried by the central server.
The knowledge encoded in the aggregated action value estimation, denoted as $\bar{Q}$, is then broadcast back to every agent for its individual policy improvement.
\textbf{Inter-agent Exploration.} The exploration-exploitation dilemma (Sec.~\ref{subsec:exploration-problem}) requires an agent to design a principled way 
% of balancing 
to balance
between exploiting its current knowledge and exploring to acquire new knowledge, which we will refer to as the 
\emph{intra}-agent exploration problem.
Similarly, {we note that in the setting of FedRL-HALE, the exploration-exploitation dilemma needs to be further considered} when the server aggregates information from all agents, i.e., when the server selects its action to interact with the environment.
This is because the actions selected by the server determine the specific state-action pairs whose value estimates $Q_n$'s gets improved {via federation} for all agents (Section \ref{sec:FedHQL}).
Therefore, a natural trade-off arises when the server selects its action:
\emph{
Should the server select actions by exploiting the current information provided by all agents, i.e., select actions which are deemed promising by all agents? Or should the server select exploratory actions for which the agents have inconsistent (i.e., high-variance) value estimations?
}
% also present in the FedRL-HALE setup - \textbf{when aggregating the knowledge from the agents to make a decision}, 
% Similarly,
% the exploration-exploitation dilemma also presents when selecting the optimal action
% in the FedRL-HALE setup -
% % also present in the FedRL-HALE setup - \textbf{when aggregating the knowledge from the agents to make a decision}, 
% \emph{should the group decision be made by exploiting the current knowledge of the group? For example, always select actions that are most appreciated by all agents, assuming that the current knowledge of the group is reliable enough. Or should the group consider a decision which seems to be sub-optimal to the group, making the assumption that the current knowledge of the group could be inaccurate?} 
This additional exploration-exploitation dilemma similarly highlights the requirement for a principled algorithm to balance the trade-off between exploiting the current knowledge of the entire group of agents and exploring to obtain new knowledge, which we will denote as \emph{inter}-agent exploration.
% This group exploration-exploitation dilemma similarly highlights the requirement for a principled algorithm to balance the trade-off between exploiting the current knowledge of the group and exploring new knowledge, which we will denote as \emph{inter}-agent exploration.

In this paper, to design the \emph{inter}-agent exploration strategy, we leverage the well-celebrated optimism-based 
% \emph{intra}-agent 
exploration strategy of upper confidence bound (UCB) (Sec.~\ref{subsec:exploration-problem}),
and develop a Federated UCB (FedUCB) algorithm which provides a principled way to balance the trade-off between exploration and exploitation when aggregating the knowledge of the group of agents.
% \textbf{why fedUCB}
% cos sequential decision-making, when selecting action, 
% promising  -> exploite

% inconsistent -> explore
% \textbf{Individual improvement}. How to cater different agents with different RL algorithm, such that they all benefit from participating in the FedRL process? For a well-trained agent, simple aggregating may not be desired and blahblah... -> we use allow the server to sample and perform fedTD to obtain feedback and use it as regularization to the individual training. if an agent is falling behind, it will result in large improvement while for a good agent, it will not affect much..

% \textbf{Complete participation}. etc.
% \textcolor{blue}{should we emphasize on this ??? e.g. not every agents of K agents must participate in this ? would sb ask: then if I am an agents, why should I sample ??? why not just wait for others?? e.g. I could just need to sample 1 time and then still send you the info and ask for updated info?? eh... even if we don't stress on this, they may still ask????} No need for all to join each round, and compatibility for new comers...

%

\section{FedHQL}\label{sec:FedHQL}
To achieve
collaborative learning of heterogeneous and black-box agents in the
setting of FedRL-HALE and
address its unique technical challenges (Section \ref{sec:technical-challenges}), we propose {a novel} \textbf{Fed}erated \textbf{H}eterogeneous \textbf{Q}-\textbf{L}earning (FedHQL) algorithm in this section. 
We will firstly discuss each component of FedHQL in detail and then present the overal algorithm in Sec.~\ref{subsec:algorithm}. 

\subsection{Federated Q-learning}
At the core of FedHQL is the federated version of Q-learning with $N$ heterogeneous and black-box agents. Each agent $\mathcal{B}_n$ \emph{independently} interacts with its own copy of the MDP using its preferred \emph{intra}-agent exploration strategy.
% such as $\epsilon$-greedy.
% To update its current estimation of action values $Q_n(s,a)$:
Each agent $\mathcal{B}_n$ updates its current estimation of action values $Q_n(s,a)$ through Q-learning as follows:
% \begin{align}\label{eq:action-seletion-n}
%     a_n \leftarrow 
% \end{align}
\begin{align}\label{eq:q-update-n}
    Q_n(s_t, a_t) \leftarrow & Q_n (s_t, a_t) + \alpha_n [\mathcal{R}(s_{t+1}, a_{t+1})  \notag \\
    &+ \gamma \max_a Q_n(s_{t+1}, a) - Q_n(s_t, a_t) ].
\end{align}
Of note, every agent is free to choose any \emph{arbitrary} policy parameterization (e.g., any policy network architecture for learning $Q_n$), training configurations for the policy network, and exploration strategy (e.g., Boltzmann exploration, $\epsilon$-greedy with different values of $\epsilon$, etc.).
% Each agent is free to choose \emph{arbitrary} network model to parameterize its policy with \emph{any} possible training configuration. 

To facilitate knowledge aggregation, we let the central server broadcast query state(s) to agents and query their estimations of the values of all actions at these candidate states (i.e., $Q_n(s,a),\forall a$). 
% To enable knowledge aggregation, we let the central server to broadcast states to agents and query their opinions of the values of available actions at the candidate states (i.e., $Q_n(s,a)$). 
Then, the server can combine the knowledge of the entire group of agents by aggregating the received action value estimations $Q_n(s,a)$'s.
% Then the server can incorporate the knowledge of the group by aggregating the received action value estimations $Q_n(s,a)$.
As mentioned in Sec.~\ref{sec:technical-challenges}, this aggregation step is faced with the exploration-exploitation challenge, calling for a principled \emph{inter}-agent exploration strategy, which we will discuss next.
% As mentioned in Sec.~\ref{subsec:exploration-problem}, this aggregation step faces the exploration-exploitation challenge, calling for a principled \emph{inter}-agent exploration strategy which we will discuss next.
% \begin{align}\label{eq:q-update-n}
%     Q^{\prime}(s_t, a_t) \leftarrow Q(s_t, a_t) &+ \alpha [\mathcal{R}(s_{t+1}, a_{t+1})  \notag \\
%     &+ \gamma \max_a Q(s_{t+1}, a) - Q(s_t, a_t) ]
% \end{align}

% \textbf{On Black-box Optimization.} We adapt value-function based approach to perform federation on the action-value function $Q(s,a)$ as defined in Sec.~\ref{subsec:q-learning}, instead of explicitly constructing such an optimization objective.
% % \textbf{policy initialization and self-learning} 
% % pass 

% \textbf{On inter-agent exploration:} We develop a federated variant of Upper Confidence Bound (UCB)-style exploration strategy, which models the randomness in the q-function of each agent. For each query state in $X^t$ in a federation round $t$, each agent is presented with a upper bound, with high confidence, on the expected value of possible actions. The details are given in Sec.~\ref{sec:fed-ucb-q}
% % pass 

\subsection{Federated Upper Confidence Bound (FedUCB)}\label{sec:fed-ucb-q}
% \textbf{UCB exploration}\\
% To provide a principled \emph{inter}-agent exploration strategy, We propose the federated Upper Confidence Bound (FedUCB) algorithm to aggregate the action values received at the server as follows:
% To proAs mentioned in Sec.~\ref{subsec:exploration-problem}, we 
% $\bar{Q}(s,a) = \frac{1}{N}\sum_{n=1}^N Q_n(s,a)$, $\mathbb{V}_{s,a} = \frac{1}{n_{s,a}} \sum_k[\overline{Q}(s,a) - Q_k(s,a)]^2$, and 
In this section, we introduce our principled strategy for \emph{inter}-agent exploration, i.e., for selecting the server actions (Section \ref{sec:technical-challenges}).

Let $b,c >0$ be constants.
For a state $s$, after all action value estimations $Q_n(s,a),\forall a$ are received, we compute the following for every action $a$:
\begin{align}
\bar{Q}(s,a) &= \frac{1}{N}\sum_{n=1}^{N} Q_n(s,a), \nonumber \\
    \mathbb{V}_{s,a} &= \frac{1}{N} \sum^{N}_{n=1}[\bar{Q}(s,a) - Q_n(s,a)]^2, \nonumber \\
    Q^{\text{UCB}}(s,a) &\triangleq \bar{Q}(s,a)  + \sqrt{\frac{2c\mathbb{V}_{s,a}}{N}} + \frac{3bc}{N}. \label{eq:fed-ucb}
\end{align}
% The FedUCB policy is then determined by choosing a action $\bar{a}$ s.t. $\bar{a} \leftarrow \arg\max_a Q^{\text{UCB}}(s,a)$. 
After that, the FedUCB policy then chooses the action $\bar{a}$ using: $\bar{a} \leftarrow \arg\max_a Q^{\text{UCB}}(s,a)$. 
{To understand the insights behind Eq.~(\ref{eq:fed-ucb}),} let us begin with the following lemma.
\begin{lemma}[Adapted from Theorem 1 in \citet{audibert2009exploration-TCS} ]\label{lemma}
% \begin{lemma}[Bernstein’s \cite{Bernstein}; Lemma ]\label{lemma}
Let $X_1,..., X_N$ be i.i.d. random variables taking their values in $[0, b]$. Let $\mu = \mathbb{E} [X_1]$ be their common expected value. Consider the empirical mean $\bar{X}_N$ and variance $\mathbb{V}_N$ defined 
% respectively 
by
\begin{equation}
\bar{X}_t=\frac{\sum_{n=1}^N X_n}{N} \text { and } \mathbb{V}_N=\frac{\sum_{n=1}^N\left(X_n-\bar{X}_N\right)^2}{N}.
\end{equation}
Then, for any $N \in \mathbb{N}$ and $\chi > 0$, with probability at least $1-3e^{-\chi}$,
\begin{equation}
\left|\bar{X}_N-\mu\right| \leq \sqrt{\frac{2 \mathbb{V}_N \chi}{N}}+\frac{3 b \chi}{N}.
\end{equation}
\end{lemma}
% \begin{lemma}[Bernstein's inequality]\label{lemma}
%     Let $U_1,\ldots,U_n\!\sim\!U$ be i.i.d random variables where almost surely $U_i \!\in \![0, b] \text{with } b>0$. Let $\mu \triangleq \mathbb{E}[U]$, and let
%     $\mathbb{V}$ be the variance of $U$.
%     % $\mathbb{V}\triangleq \frac{1}{n} \sum[\bar{Q}(s,a) - Q_k(s,a)]^2$.
%     Then, for any $\delta \in (0,1)$, we have 
%     \begin{equation*}
%       \mathbb{P}\Big(\frac{1}{n}\sum_{i=1}^n U_i - \mu < \frac{b}{n}\log(\frac{1}{\delta}) + \sqrt{\frac{2\mathbb{V}\log (1/\delta)}{n}}\Big) \geq 1 - \delta .
%     \end{equation*}
% \end{lemma}
Lemma~\ref{lemma} has been used by variance-based UCB algorithms for multi-armed bandits, and its proof is provided in \citet{audibert2009exploration-TCS}.
% form of Bernstein's inequality \cite{Bernstein} which provides a high-probability confidence interval for the empirical mean of an i.i.d. sequence. 
Inspired by Lemma~\ref{lemma}, 
% define $U_i= Q_i$, $\bar{Q} = \frac{1}{N}\sum_{i=1}^N Q_i$ and $\mathbb{V} = \frac{1}{N} \sum_i[\bar{Q} - Q_i]^2$.
we have the following:
% and \cite{audibert2009exploration-TCS} and discussion
\begin{theorem} %[Fed-Q-UCB]
\label{theorem}
    For any $s$ and $a$, assume that $\{Q_n(s,a)\}^N_{n=1}$ are i.i.d. samples drawn from a distribution whose expectation is the optimal Q function at $s$ and $a$: $Q^*(s,a)$.
    Denote $\mu_{s,a} \triangleq \mathbb{E}[Q_n(s,a)] = Q^*(s,a)$.
    Also assume that $Q_n(s,a)\in[0,b],\forall s,a,n$.
    % Consider each state $s$ at the Server, let $\mu_{s,a} \triangleq \mathbb{E}[Q_n(s,a)]$ be the common expected value of action $a$ conditioned on state $s$. 
    Then for any $c > 0$,
    % let $c = \log \frac{1}{\delta}$, 
    with probability at least $1 - 3e^{-c}$, we have
    % Consider each state $s$ at the Server, let $\mu_{s,a} \triangleq \mathbb{E}[Q_n(s,a)]$ be the common expected value of action $a$ conditioned on state $s$. Then for $c > 0$,
    % % let $c = \log \frac{1}{\delta}$, 
    % with probability at least $1 - 3e^{-c}$, we have
    \begin{equation*}
        % |\bar{Q}(s,a) - \mu_{s,a}| \leq \frac{bc}{n_{s,a}} + \sqrt{\frac{2c\mathbb{V}_{s,a}}{n_{s,a}}}
        |\bar{Q}(s,a) - \mu_{s,a}| \leq  \sqrt{\frac{2c\mathbb{V}_{s,a}}{N}} + \frac{3bc}{N}.
    \end{equation*}
\end{theorem}
\begin{proof}
The proof follows directly from Lemma~\ref{lemma} 
% \cite{audibert2009exploration-TCS},
by setting $X_n = Q_n(s,a)$, 
% $t=N$, 
$\bar{Q}(s,a) = \frac{1}{N}\sum_{n=1}^{N} Q_n(s,a)$,
and \\
$\mathbb{V}_{s,a} = \frac{1}{N} \sum^N_{n=1}[\bar{Q}(s,a) - Q_n(s,a)]^2$.
% % which gives a high-probability confidence interval for the empirical mean of an i.i.d sequence. \textbf{say sth about i.i.d. estimation of Q(s,a)}.
% \textcolor{red}{With Lemma~\ref{lemma}, the choice $U_i = Q_i$, $\forall \delta \in (0,1)$, and the following}: 
% \begin{align*}
%     \mu_{s,a} &\triangleq \mathbb{E}[Q_i(s,a)], \quad \bar{Q}(s,a) \triangleq \frac{1}{N}\sum_{i=1}^N Q_i\\
%     c &= \log \frac{1}{\delta} , \qquad \quad \mathbb{V}_{s,a}  \triangleq \frac{1}{N}\sum[\bar{Q}(s,a)-Q_i(s,a)]^2,
% \end{align*}
% we obtain the concentration result as follows, with probability at least $1-e^{-c}$,
% \begin{align*}
%     \bar{Q}(s,a) - \mu < \frac{bc}{N} + \sqrt{\frac{2\mathbb{V}_{s,a}c}{N}}
% \end{align*}
% Let $\Phi = \frac{bc}{N} + \sqrt{\frac{2\mathbb{V}_{s,a}c}{N}}$. We have
% which yields the theorem since $||$.
\end{proof}
The assumption that $\{Q_n(s,a)\}^N_{n=1}$ are i.i.d. samples from a distribution whose expectation is $Q^*(s,a)$ is justified because every agent $\mathcal{B}_n$ aims to independently optimize its $Q_n$ such that it can estimate the optimal Q function $Q^*$.
The assumption that $Q_n(s,a)\in[0,b]$ is also easily satisfied as long as the reward function is non-negative and upper-bounded.
For every pair of $(s,a)$, Theorem~\ref{theorem} provides a high-probability upper bound on the difference between the aggregated action value $\bar{Q}(s,a)$ and the optimal action value, resulting in the following corollary:
\begin{corollary}[FedUCB]\label{corollary}
   % Consider each state $s$ at the Server, let $\mu_{s,a} \triangleq \mathbb{E}[Q_n(s,a)]$ be the common expected value of action $a$ conditioned on state $s$. Then for $c > 0$,
   %  % let $c = \log \frac{1}{\delta}$, 
   %  with probability at least $1 - 3e^{-c}$, we have
    Under the same assumptions and notations as Theorem \ref{theorem},
    for any $c > 0$, with probability at least $1 - 3e^{-c}$, we have
    \begin{equation*}
        \mu_{s,a} \triangleq Q^*(s,a) \leq Q^{\text{UCB}}(s,a) \triangleq \bar{Q}(s,a)  + \sqrt{\frac{2c\mathbb{V}_{s,a}}{N}} + \frac{3bc}{N}.
    \end{equation*}
    % \begin{equation*}
    %     \mu_{s,a} \leq Q^{\text{UCB}}(s,a) \triangleq \bar{Q}(s,a)  + \sqrt{\frac{2c\mathbb{V}_{s,a}}{N}} + \frac{3bc}{N}
    % \end{equation*}
\end{corollary}
% Corollary~\ref{corollary} is a direct result of Theorem~\ref{theorem}.
Corollary~\ref{corollary} suggests that the optimal value of action $a$ at state $s$, $Q^*(s,a)$, is upper-bounded by $ Q^{\text{UCB}}(s,a)$ defined in Eq.~\eqref{eq:fed-ucb} with high confidence. 
% With Corollary~\ref{corollary}, the expected value of action $a$ at state $s$ is thus upper-bounded, with high confidence, by $ Q^{\text{UCB}}(s,a)$ defined in Eq.~\eqref{eq:fed-ucb}. 
Inspired by Corollary~\ref{corollary}, 
we develop our practical FedUCB algorithm for the knowledge aggregation in FedRL-HALE, which firstly calculates (for any $s,a$):
\begin{align}
    \bar{Q}(s,a) &= \frac{1}{N}\sum_{n=1}^{N} Q_n(s,a), \label{eq:Q-mean}\\
    {Q}^{\text{std}}(s,a) &= \sqrt{\frac{1}{N} \sum_{n=1}^{N}[\bar{Q}(s,a) - Q_n(s,a)]^2}, \label{eq:Q-std}\\
    Q^{\text{UCB}}(s,a) & \simeq \underbrace{\bar{Q}(s,a)}_{\text{exploitation}}+ \lambda  \underbrace{Q^{\text{std}}(s,a)}_{\text{exploration}}, \label{eq:fed-AUCB}
\end{align}
and then selects the server action by:
\begin{align}\label{eq:fed-action-selection}
    \bar{a}_t \leftarrow \arg\max_a Q^{\text{UCB}}(s_t,a).
\end{align}
Note that the selected action $\bar{a}_t$ will determine the state-action pair whose value estimate $Q_n(s_t,\bar{a}_t)$ gets improved for all agents (details in Section \ref{subsec:fed:td}).
Therefore, the FedUCB algorithm \eqref{eq:fed-action-selection} achieves both \emph{exploitation} of the knowledge from all agents by preferring actions deemed promising by all agents (i.e., large $\bar{Q}(s,a)$) and \emph{exploration} of those actions for which the value estimations from all agents are inconsistent (i.e., large $Q^{\text{std}}(s,a)$).
The degree of exploration is controlled by the parameter $\lambda$, which we will refer to as \emph{inter}-agent exploration coefficient, such that a larger $\lambda$ encourages the selection of more exploratory actions.

% The second half of Eq.~\eqref{eq:fed-AUCB} adds exploration, with the degress of exploration being controlled by the hyper-parameter $\lambda$ we will refer to as \emph{inter}-agent exploration coefficient. 
% % where $\lambda$ is the \emph{inter}-agent exploration coefficient that balances the trade-off between exploiting current knowledge of the group and exploring new knowledge.
% With larger $\lambda$, Eq.~\eqref{eq:fed-AUCB} motivates FedUCB to select more exploratory actions for which the agents have inconsistent value estimations (i.e., high-variance).
% % when the agents have inconsistent value estimations.
% While a smaller value of $\lambda$ encourages FedUCB to select actions which are deemed promising by all agents.
% \begin{align}\label{eq:fed-action-selection}
%     \bar{a}_t \leftarrow \arg\max_a Q^{\text{UCB}}(s_t,a).
% \end{align}

\subsection{Federated Temporal Difference (FedTD)}
\label{subsec:fed:td}
With the FedUCB derived above , the server is able to optimistically select an action that leads to 
% highest 
high
returns with high probability. 
Inspired by \citet{fan2021fault-FRL}, we let the server operate in another separate copy of the underlying MDP and execute the selected action $\bar{a}$, hence generating a new sample $(s_t,\bar{a},s_{t+1},r_t)$.
This new sample will then be used to perform a federated version of Temporal Difference (FedTD) learning, which is defined as follows:
\begin{align}\label{eq:fed-TD}
    \bar{Q}(s_t,\bar{a}_t) & \leftarrow \bar{Q}(s_t,\bar{a}_t) + \nonumber \\ & {\alpha_{\text{s}}} \Big(r_t + \gamma \max_{b} \bar{Q}(s_{t+1},b) - \bar{Q}( s_t,\bar{a}_t)\Big)
\end{align}
where $\bar{Q}(s_t,a_t)$, defined in Eq.~\eqref{eq:Q-mean}, represents the agents' current estimation of the action value at time $t$, {$\alpha_{\text{s}}$ is the learning rate of FedTD}, and $r_t + \gamma \max_b\bar{Q}(s_{t+1},b)$ is the Temporal Difference (TD) target \cite{sutton2018reinforcement-RL}.
%
% here is why we want the server to sample
%
To calculate $\bar{Q}(s_{t+1},b)$ which is used in Eq.~\eqref{eq:fed-TD}, the server sends the state $s_{t+1}$ to all agents and collects the value estimations of all actions at $s_{t+1}$ from all agents, after which $\bar{Q}(s_{t+1},b)$ can be again calculated using Eq.~\eqref{eq:Q-mean}.
If the action executed $\bar{a}_t$ is not optimal, then it gets penalized by receiving a smaller $r_t$. 
Essentially, FedTD regularizes the action selected by FedUCB, $\bar{a}_t$, at each time step $t$ by updating $\bar{Q}(s_t, \bar{a}_t)$ with TD target $r_t + \gamma \max_b\bar{Q}(s_{t+1},b)$. 
% If
% state why we want the server to sample... which is to get reward (feedback) from the environment to our federated decision 

At each time step, after the FedTD target \eqref{eq:fed-TD} is calculated, it will be broadcast to all agents so that they can use it to update their action value estimations 
% $Q_n(s_t,a_t)$'s, 
$Q_n$'s,
which is discussed next.
% To implement Eq.~\eqref{eq:fed-TD}, we will also need to update agents' action value estimations $Q_n(s_t,a_t)$ at each time step, which is discussed next.

\subsection{Individual Improvement}
One objective of FedRL-HALE is to improve the sample efficiency of the individual participating agents.
To achieve this,
after the FedTD target $\bar{Q}(s_t,\bar{a}_t)$ is updated following Eq.~\eqref{eq:fed-TD}, 
we let the server broadcast the updated $\bar{Q}(s_t,\bar{a}_t)$ back to all agents.
% we let the server broadcast the updated $\bar{Q}(s_t,\bar{a}_t)$ back to each agent.
An agent $\mathcal{B}_n$ will then update its own action value estimation 
% $Q_n(s,a)$ 
$Q_n$ 
using the following regression loss:
\begin{align}
    \mathcal{L}_n \triangleq \|\bar{Q}(s_t,\bar{a}_t) - Q_n(s_t,\bar{a}_t)\|^2 \label{eq:individual-improvement-loss}
\end{align}
This loss serves as a regularizer that periodically updates agent $\mathcal{B}_n$'s parameter $\omega_n$ by
\begin{align}\label{eq:individual-improvement-update}
    \omega_n \leftarrow \omega_n - \tilde{{\alpha}}_n \nabla\mathcal{L}_n
\end{align}
where $\tilde{{\alpha}}_n$ is a step-size hyper-parameter.
This loss essentially helps the agent to improve its knowledge about action $\bar{a}_t$ at state $s_t$ using the knowledge aggregated by FedUCB and updated by FedTD. {In FedHQL, each agent performs $\kappa$ steps of gradient updates in Eq.~\eqref{eq:individual-improvement-update}.}
% $\tilde{\eta}_n$ is a step-size hyper-parameter that controls how significant the update in~\eqref{eq:individual-improvement-update} should affect agent $\mathcal{B}_n$'s parameters $\omega_n$. 
% % say sth about $\eta_n$, the individual improvement rate, which controls how fast an agent responds to the federated decision-making. 
% if an agent is performing near optimal, it may opts for smaller rate.. blahblah..

\subsection{FedHQL Algorithm}\label{subsec:algorithm}
\begin{figure}[b]
    \centering
    % \noindent
    % \makebox[\textwidth]{\includegraphics[width=5.5in, height=1.2in] {./plots/EXP1.pdf}}
    \includegraphics[width=3in] {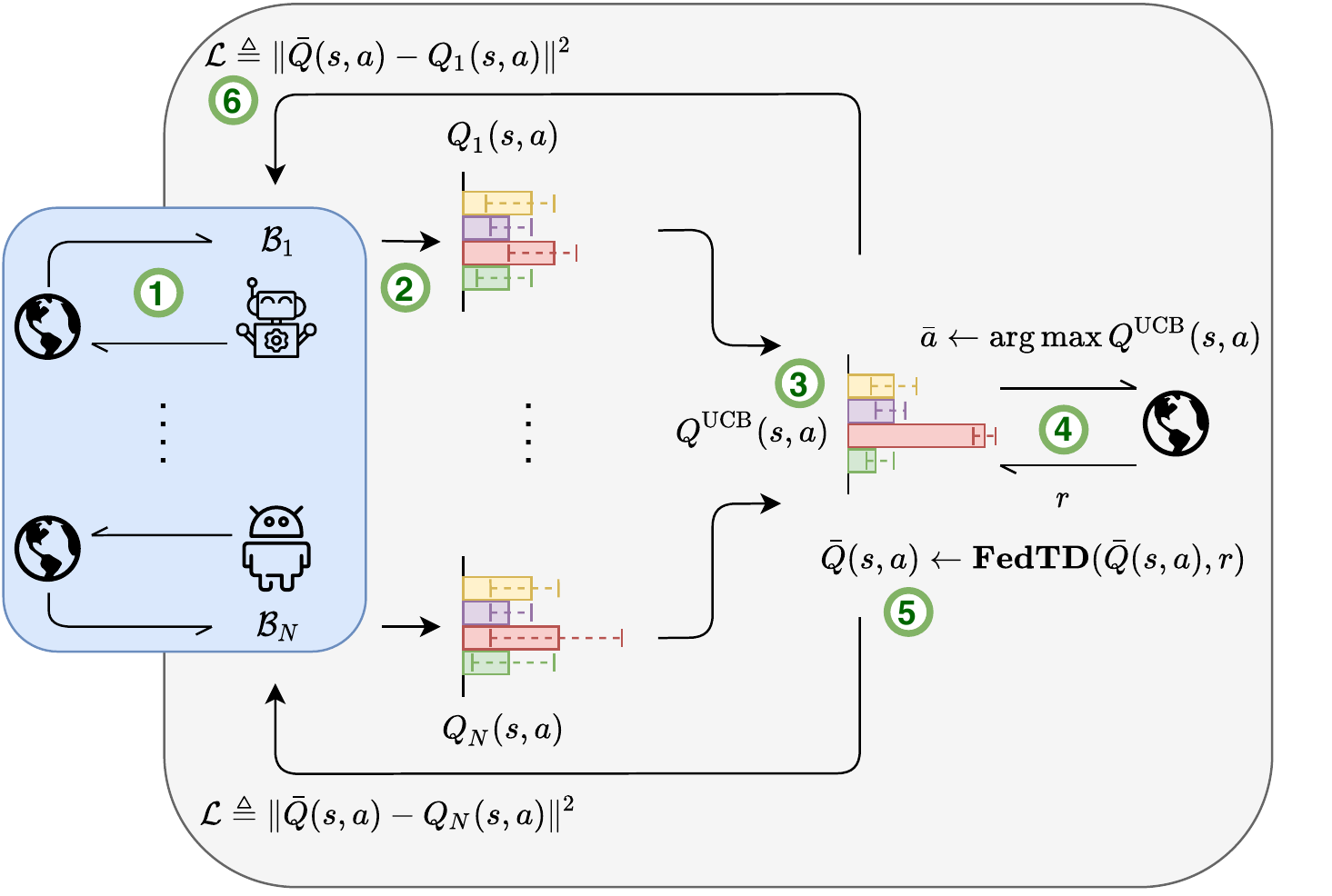}
    % \vspace{-8pt}
    \caption{Graphical illustration of FedHQL.}
    \label{fig:fedHQL}
\end{figure}
We are now ready to present the complete algorithm of FedHQL.
% , whose 
% pseudocode shown in Algorithm~\ref{alg:1.a}\&\ref{alg:1.b} are deferred to Appendix.
Fig.~\ref{fig:fedHQL} gives an illustration of FedHQL, and we describe the details of every step below.
% illustrates the FedHQL algorithm the details of which we describe below. 
% Fig.~\ref{fig:fedHQL} illustrates the FedHQL algorithm the details of which we describe below. 
One federation round $t$ in FedHQL starts by the server broadcasting a query state $s_t$ to all agents:

$\bullet$ \quad \textbf{\textcircled{{1}}\textcircled{{2}} (by Agents):} 
At any point of FedHQL, every agent $\mathcal{B}_n$ \emph{independently} performs self-learning according to Eq.~\eqref{eq:q-update-n} with any arbitrary choice of policy parameterization, training configurations, and intra-exploration strategy (step \textcircled{{1}}). When a query state $s_t$ is received from the server, 
agent $\mathcal{B}_n$ sends to the server $Q_n(s_t,a),\forall a$, i.e., its current estimation of the action values at the query state $s_t$ for all actions $a$'s (step \textcircled{{2}}).
% it sends back its estimation of action values $Q_n(s_t,a_t)$ (step \textcircled{{2}}).

$\bullet$ \quad \textbf{\textcircled{{3}}\textcircled{{4}}\textcircled{{5}} (by Central Server):}
After receiving the action value estimations from all agents, for every action $a$, the server computes $\bar{Q}(s_t,a)$, ${Q}^{\text{std}}(s_t,a)$, and $Q^{\text{UCB}}(s_t,a)$ according to Eq.~\eqref{eq:Q-mean}, Eq.~\eqref{eq:Q-std}, and Eq.~\eqref{eq:fed-AUCB} (step \textcircled{{3}}). 
% After receiving action value estimations from agents, the server computes $\bar{Q}(s_t,a_t)$, ${Q}^{\text{std}}$, and $Q^{\text{UCB}}(s_t,a_t)$ according to Eq.~\eqref{eq:Q-mean}, Eq.~\eqref{eq:Q-std}, and Eq.~\eqref{eq:fed-AUCB} (step \textcircled{{3}}). 
An action $\bar{a}$ is then selected and executed by the server according to Eq.~\eqref{eq:fed-action-selection}, which generates the sample $(s_t,\bar{a}_t,s_{t+1},r_t)$ (step \textcircled{{4}}). 
Next, to prepare for the calculation of the FedTD target \eqref{eq:fed-TD}, the server broadcasts the state $s_{t+1}$ to all agents, and then receives from all agents their value estimations at the state $s_{t+1}$ for all actions, i.e., $\{Q_n(s_{t+1},a),\forall a\}^N_{n=1}$.
This allows the server to calculate $\bar{Q}(s_{t+1},a),\forall a$ again using Eq.~\eqref{eq:Q-mean}.
Then, the server performs FedTD to update $\bar{Q}(s_t,\bar{a}_t)$ according to Eq.~\eqref{eq:fed-TD} and broadcasts the updated $\bar{Q}(s_t,\bar{a}_t)$ to all agents (step \textcircled{{5}}).

$\bullet$ \quad \textbf{\textcircled{{6}} (by Agents):}
Whenever $\bar{Q}(s_t, \bar{a}_t)$ is received, agent $\mathcal{B}_n$ performs the individual improvement for $\kappa$ steps to update its own
$Q_n$
according to the loss function defined in Eq.~\eqref{eq:individual-improvement-loss} and gradient updates in Eq.~\eqref{eq:individual-improvement-update} (step \textcircled{{6}}). 

{This federation process (\textcircled{{2}}-\textcircled{{6}}) is repeated until $s_t$ is a terminal state or the maximum federation time horizon $H_\text{fed}$ is reached. After a federation round, every agent continues to repeat the independent self-learning (step \textcircled{{1}}) followed by the federation (\textcircled{{2}}-\textcircled{{6}}), until all the budgets for environment interactions are exhausted.} Of note, the above process can be implemented {in asynchronous batches} 
using the reliable broadcast protocol from distributed computing \cite{chang1984reliable}.
Due to space constraints, the pseudocode of FedHQL is
deferred to Appendix~\ref{appendix:sec:algorithm-pseudocode}.

\section{Evaluation}\label{sec:experiments}
In this section, we empirically verify the effectiveness of the FedHQL algorithm
% by demonstrating its effectiveness at boosting the sample efficiency of individual agents 
using the classical control tasks of CartPole and LunarLander from the OpenAI gym environment \cite{brockman2016openai-gym}.
% 
% tasks from the OpenAI gym environment \cite{brockman2016openai-gym}, including the classical control tasks of CartPole balancing and LunarLander as shown in Fig.~\ref{fig:envs}. 
We design our experiments with respect to the two objectives of FedRL-HALE 
% (Sec.~\ref{sec:problem-setting}):
(Sec.~\ref{subsec:problem-formulation}):
% \begin{enumerate}
%     \item Given a fixed budget on the total number of interactions with the whole system,
%     % samples for the whole system, 
%     does FedHQL improve the average performance of all agents?
%     \item If an agent participates in the FedRL-HALE setup, does the FedHQL algorithm improve the sample efficiency of the agent with high probability?
%     % of individual participating agent through reducing the number of samples of agent-environment interactions compared to the independent self-learning?
% \end{enumerate}
\textbf{I}. Given a fixed budget on the total number of interactions with the whole system,
    % samples for the whole system, 
does FedHQL improve the average performance of all agents?    
\textbf{II}. If an agent participates in the FedRL-HALE setup, does the FedHQL algorithm improve the sample efficiency of the agent with high probability?

% \begin{figure}
%     \centering
%     \includegraphics[width=3.3in]{imgs/envs.pdf}
%     \caption{Simulation of the RL control tasks. \emph{(left)} Cart Pole balancing task in which the RL agent learns to balance the pole attached to a cart by applying horizontal forces to the cart. The reward of a perfect trajectory is capped at 500. \emph{(right)} Lunar Rocker Landing task which is a classic rocket trajectory optimization problem. The reward of a perfect trajectory of landing the rocket is capped at 200.}
%     \label{fig:envs}
% \end{figure}

\subsection{Experimental Settings}
To the best of our knowledge, this is the first FedRL algorithm that is able to aggregate knowledge from \emph{heterogeneous} and \emph{black-box} agents.
% To the best of our knowledge, this is the first FedRL algorithm that attempts to aggregate knowledge from heterogeneous and black-box agents.
We thus compare the performance of 
% training them 
the agents
\emph{with} FedHQL and \emph{without} federation (independent self-learning) which we refer to as DQN (w.o. Fed). 
We initialize $N=5$ heterogeneous 
% black-box 
agents and use the default implementation of classical deep Q-learning (DQN) from the \emph{stable baselines} library \cite{raffin2021-stable-baselines3} as the backbone training algorithm for each agent.
% We initialize $N=5$ heterogeneous black-box agents and use the \emph{stable baselines} \cite{raffin2021-stable-baselines3} library's default implementation of classical deep Q-learning (DQN) as the backbone training algorithm for each agent.

\begin{table}[ht]
\centering
  \caption{Configurations of $N=5$ agents}
  \label{tab:configuration}
  \resizebox{0.45\textwidth}{!}{%
  \begin{tabular}{cccc}\toprule
    \textit{Agent No.} & \textit{Network} &  \textit{Learning rates} &  \textit{Intra-exploration coefficient} \\ \midrule
    1 & 64x64 (Tanh) & 0.005 & 0.01 \\
    2 & 128x128 (ReLU) & 0.01 & 0.1 \\
    3 & 32x32 (Tanh) & 0.01 & 0.05 \\
    4 & 16x16 (ReLU) & 0.02 & 0.01 \\
    5 & 8x8x8 (ReLU) & 0.001 & 0.01 \\ \bottomrule
  \end{tabular}
  }
\end{table}

% \begin{figure*}[h]
%     \centering
%     \includegraphics[width=6.7in] {imgs/exp1_agent_Cartpole.pdf} %ignore plots for faster compiling
%     % \includegraphics[width=6.7in, height=1.8in] {plots/EXP1.pdf} %ignore plots for faster compiling
%     \caption{Individual learning curves of each agent running FedHQL ($N=5$ agents) with different inter-agent exploration coefficients (FedHQL ($\lambda\!=\!0,1,3,5,10$) against independent training (DQN (w.o. Fed)) on the CartPole environment.}
%     \label{fig:exp1-cartpole}
% \end{figure*}
% \begin{figure*}[h]
%     \centering
%     \includegraphics[width=6.7in] {imgs/exp1_agent_LunarLander.pdf}
%     \caption{Individual learning curves of each agent running FedHQL ($N=5$ agents) with different inter-agent exploration coefficients (FedHQL ($\lambda\!=\!0,1,3,5,10$) against independent training (DQN (w.o. Fed)) on the LunarLander environment.
%     }
%     \label{fig:exp1-lunarlander}
% \end{figure*}
To study the sensitivity of FedUCB to the \emph{inter}-agent exploration coefficient $\lambda$, we run our FedHQL with different values of $\lambda = 0,1,3,5,10 $. 
The $\epsilon$-greedy algorithm (Sec.~\ref{subsec:exploration-problem}) is chosen to be the \emph{intra}-agent exploration strategy for all agents.
In all of our experiments, agents are configured differently and the server has no knowledge about any agent.
The configurations of the agents are depicted in Table~\ref{tab:configuration}, in which in the second column, 64x64 (Tanh) denotes a two-layer network with 64 neurons in each layer and the Tanh activation function. 
The policy networks (architectures and activation functions), learning rates, and intra-exploration coefficients of different agents are \emph{arbitrarily} set to different values to simulate heterogeneity of agents in practical applications where  agents with various computational budgets may have different assessments of the task.

\begin{figure*}[ht]
    \centering
    \includegraphics[width=6.7in] {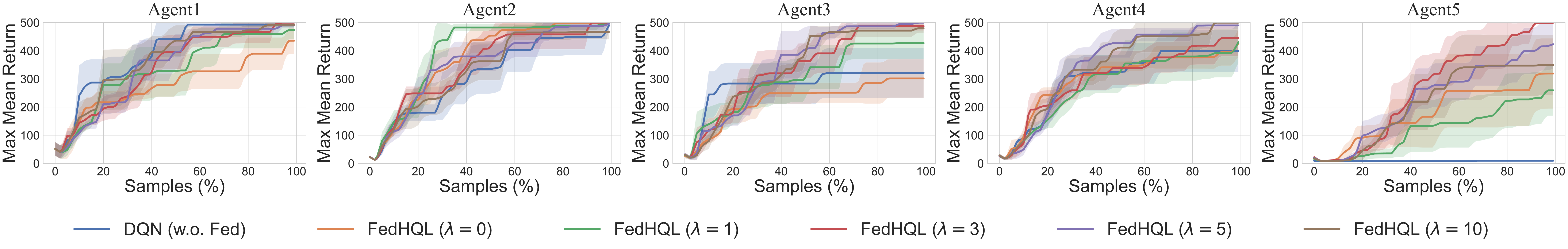} %ignore plots for faster compiling
    % \includegraphics[width=6.7in, height=1.8in] {plots/EXP1.pdf} %ignore plots for faster compiling
    % \caption{Individual learning curves of each agent running FedHQL ($N=5$ agents) with different inter-agent exploration coefficients (FedHQL ($\lambda\!=\!0,1,3,5,10$) against independent training (DQN (w.o. Fed)) on the CartPole environment.}
    % \label{fig:exp1-cartpole}
\par\medskip
    % \centering
    \includegraphics[width=6.7in] {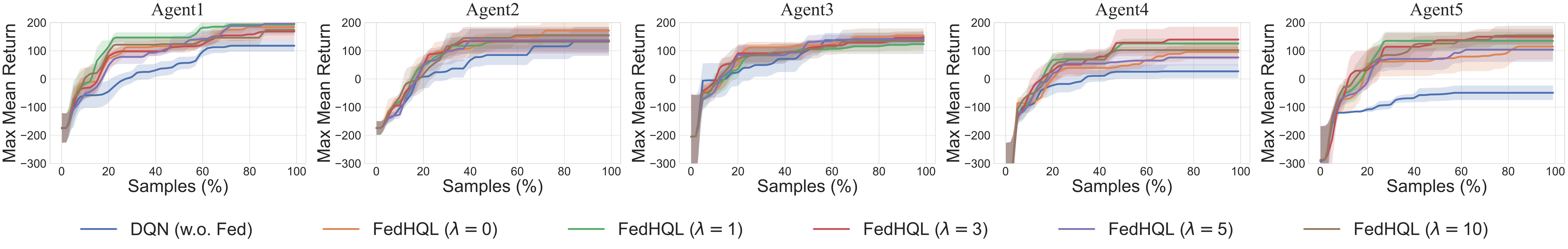}
    \caption{Individual learning curves of ($N=5$) agents running FedHQL with different inter-agent exploration coefficients (FedHQL ($\lambda\!=\!0,1,3,5,10$) against independent training (DQN (w.o. Fed)) on the CartPole (top) and LunarLander (bottom) environments.
    }
    \label{fig:exp1-individual}
\end{figure*}

Pertaining to the budget of each agent, 
we stipulate a fixed budget of maximum $2\times10^6$ environment interactions for CartPole and $4\times10^6$ environment interactions for LunarLander.
% to stress on that server samples very little?
%
%
% Note that although our FedHQL would require the central server to also sample data, 
% % we find that once the hyper-parameters of FedHQL is set properly, 
% the total number of samples required by the server would not exceed 10\% of the aforementioned budget in order for FedHQL to work as desired\footnote{In the experiments of this paper, the server would only sample about 8\% of the total budget for CartPole and 4\% of the total budget for LunarLander, respectively.}, which we believe could be well accepted in practical implementation. 
%
% maybe say this in obj 1 -> to stress on that server samples very little yet it improves system performance significantly
%
% to stress on that server samples very little?
The discount factor is set to be 0.999 for CartPole and 0.990 for LunarLander.
{
The hyper-parameters of FedHQL are set as follows: we use number of gradient steps $\kappa$=64, maximum federation time horizon $H_\text{Fed}$=16, FedTD learning rate $\alpha_\text{s}$=0.05, and we set the batch size of the queries to be 128. The number of self-learning steps is set to be 5k for CartPole and 10k for LunarLander. And we update the target network of the backbone DQN algorithm every 1k steps for CartPole and 2k steps for LunarLander.} Other experimental details follow the default settings of \emph{stable baselines} \cite{raffin2021-stable-baselines3} and are kept the same for every set of experiments across different runs for a fair comparison.
In all experiments, we report the metric \emph{max mean return}, in which we compute the average/mean performance of every 10 test episodes during training and report the max value among them. 
Each experiment is repeated with 5 independent runs with 80\% bootstrap confidence intervals using random seeds from 0 to 4.

\subsection{Efficacy in Improving System Welfare}\label{subsec:exp-object-1}
We firstly investigate the efficacy of FedHQL in improving the system welfare, i.e., the first objective of FedRL-HALE
(Sec.~\ref{subsec:problem-formulation}).
% (Sec.~\ref{sec:problem-setting}). 
In particular, given the fixed budget of each agent, we examine the average performance of agents versus the average consumption of the budget per agent.
The results in both tasks are plotted in Fig.~\ref{fig:exp1}.
Of note, the sample cost incurred at the server is also included in the computation of the budget of an agent in FedHQL following~\eqref{eq:objective-1}.
% in the first objective.
The figures show that FedHQL with different choices of \emph{inter}-agent exploration coefficients, FedHQL ($\lambda = 0,1,3,5,10$), significantly improves the average performance per agent over that of independent self-learning, DQN (w.o. Fed).
For example, in the LunarLander task, an agent is expected to consume at least $40\%$ of its budget (i.e., total $1.6\text{m} = 4\times 10^6 \times0.4$ interactions) on average to receive positive returns while an agent in FedHQL ($\lambda=1$) can achieve a performance close to 100 using only about $20\%$ of its budget (i.e., total $0.8\text{m} = 4\times 10^6 \times0.2$ interactions).
% need this line?
%
%
%
The results also suggest that FedHQL is less sensitive to the choice of $\lambda$ in the LunarLander task compared to the CartPole task, which we think is because 
any degree of \emph{inter}-agent exploration (i.e., collaboration among agents) would help significantly in the more difficult task of LunarLander.
% any degree of exploration would help greatly in the more difficult task (LunarLander).
\begin{figure}
    \centering
    \includegraphics[width=3in]{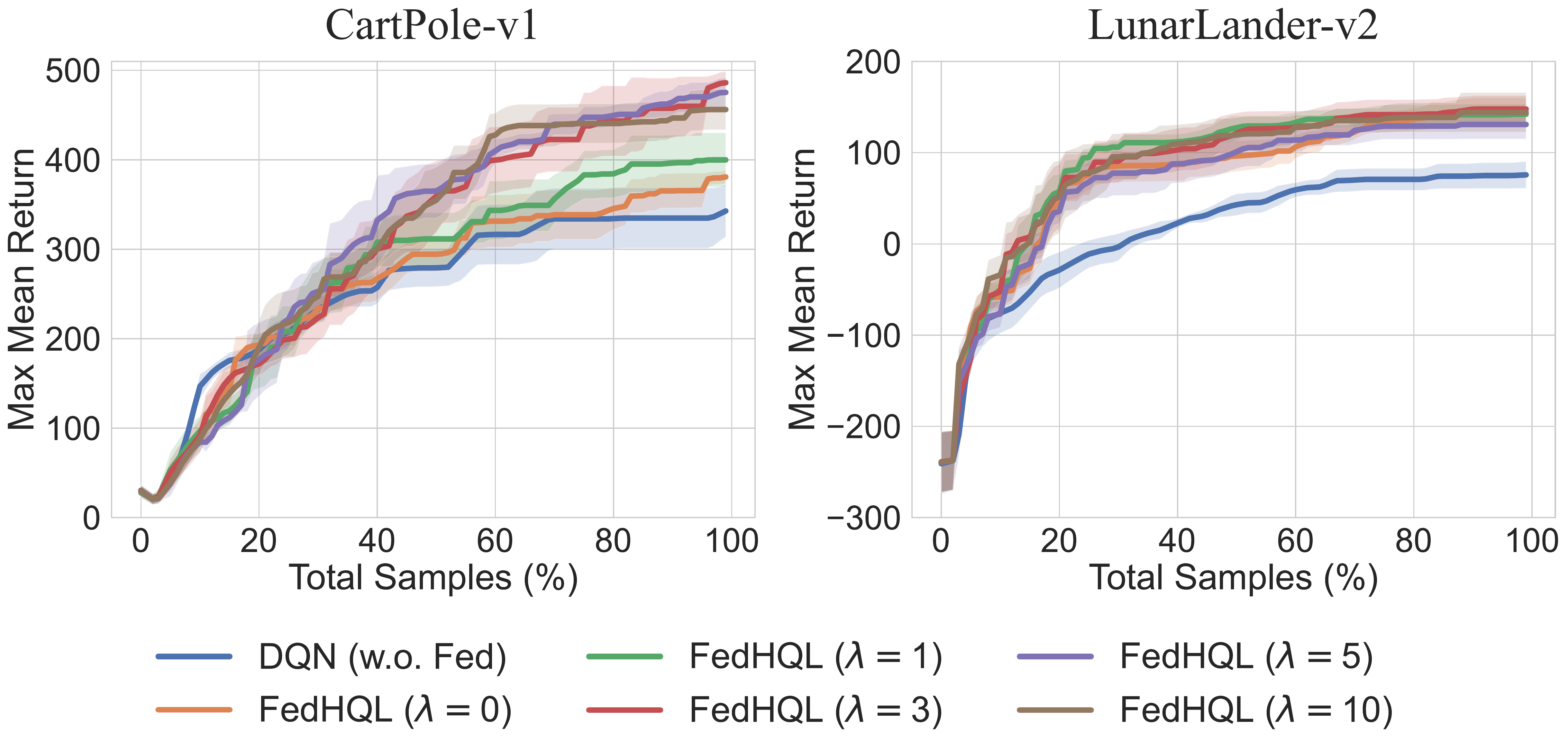}
    \caption{Learning curves of FedHQL ($N=5$ agents) with different inter-agent exploration coefficients (FedHQL ($\lambda\!=\!0,1,3,5,10$) against independent self-learning of agents (DQN (w.o. Fed)) on the CartPole (\emph{left}) and LunarLander (\emph{right}) tasks. }
    \label{fig:exp1}
\end{figure}
\subsection{Effectiveness in Improving Individual Agents}\label{subsec:exp-object-2}
Next, we verify the effectiveness of FedHQL in boosting the sample efficiency of each participating agent with high probability, i.e., the second objective of FedRL-HALE 
(Sec.~\ref{subsec:problem-formulation}).
% (Sec.~\ref{sec:problem-setting}). 
% In this set of experiment,
In this set of experiments, we report the performance versus the consumed samples of each individual agent (Agent1 to Agent5) during the corresponding training process.
Fig.~\ref{fig:exp1-individual} shows the results for the CartPole (top) and LunarLander (bottom) tasks. 
The figures show that for both tasks, FedHQL is able to boost the sample efficiency for most agents.
For example, Agent1 needs to consume over $60\%$ of its budget to reach a performance of 100 in LunarLander with independent self-learning, DQN (w.o. Fed).
If Agent1 participates in FedHQL ($\lambda = 1 $) with the other four agents, its performance can be boosted to around 150 with just $20\%$ 
% consumption 
of its budget.
Similar results can be observed from the other four agents on the LunarLander environment.

Interestingly, FedHQL ($\lambda=0$) renders the performance of Agent1 and Agent3 slightly worse than their corresponding independent self-learning (DQN (w.o. Fed)).
This is because Agent5 (whose policy network is poorly parameterized) fails to learn the CartPole task within the given budget, causing inaccurate knowledge (estimation of the action values) of Agent5 to be sent to the server.
As discussed in Sec.~\ref{sec:fed-ucb-q}, FedUCB with $\lambda = 0$ encourages the server to strongly exploit the group knowledge $\bar{Q}$ which is inaccurate due to the misleading knowledge contributed by Agent5.
As a result, the performances of Agent1 and Agent3 in FedHQL ($\lambda = 0$) deteriorate slightly while their performance is still improved with FedHQL ($\lambda = 1,3,5,10$), which corroborates our theoretical insights on FedUCB (Sec.~\ref{sec:fed-ucb-q}).
Moreover, it is worth noting that FedHQL ($\lambda=0,1,3,5,10$) significantly improves the performance of Agent5 from being completely unlearnable to 
solving the task with nearly maximum return. 
Similar experimental results with more $N = 10$ agents are given in Appendix~\ref{appendix:exp-more-agents}.

\section{Conclusion}\label{sec:discussion}
% Federated Reinforcement Learning (FedRL) is a promising learning paradigm that improves the sample efficiency of RL agents without transmitting their raw experience data. 
% The study of FedRL would speed up the practical development of RL in real-world applications where the cost of sampling from the environment is slow and expensive. 
In this work, we introduce a practical formulation of FedRL with heterogeneous and black-box agents and discuss the unique challenges posed by this novel setting. We have presented principled solutions to these challenges and proposed the FedHQL algorithm.
% , the effectiveness of which is empirically verified. 
% ----- below ---
% Although various aspects of FedRL agents may be heterogeneous, our current algorithm of FedHQL assumes that the participating agents follow the Q-learning family of RL methods.
% Therefore, it would be a promising future direction to design FedRL algorithms that allow the heterogeneous agents to choose \emph{arbitrary} RL methods such as Q-learning and Policy Gradients. 
% above -----
Due to the difficulty in analyzing the convergence of Q-learning and that heterogeneous agents may utilize arbitrarily non-linear policy parameterization, the convergence of FedHQL and its correlation with the number of agents are not studied in this work, which we consider promising future directions.
% difficulty in analyzing the convergence of Q-learning itself and this  heterogeneous parameterization,

% It would be a promising future direction to 
% do so and 
% study how the performance of FedHQL scales with the number of agents.

% It would be a promising future direction to design FedRL algorithms to allow heterogeneous agents to choose \emph{arbitrary} RL methods such as Q-learning and Policy Gradients. 

% % ----- consider below carefully
% In addition, another interesting future topic is to reduce the cost of communication between the agents and the server, potentially via distributed computing methods such as parallelization and asynchronization.
% % --------------------
\bibliography{example_paper}

\begin{thebibliography}{34}
\providecommand{\natexlab}[1]{#1}
\providecommand{\url}[1]{\texttt{#1}}
\expandafter\ifx\csname urlstyle\endcsname\relax
  \providecommand{\doi}[1]{doi: #1}\else
  \providecommand{\doi}{doi: \begingroup \urlstyle{rm}\Url}\fi

\bibitem[Audibert et~al.(2009)Audibert, Munos, and
  Szepesv{\'a}ri]{audibert2009exploration-TCS}
Audibert, J.-Y., Munos, R., and Szepesv{\'a}ri, C.
\newblock Exploration--exploitation tradeoff using variance estimates in
  multi-armed bandits.
\newblock \emph{Theoretical Computer Science}, 410\penalty0 (19):\penalty0
  1876--1902, 2009.

\bibitem[Bajaj et~al.(2021)Bajaj, Arora, and Hasan]{bajaj2021-blackbox-review}
Bajaj, I., Arora, A., and Hasan, M.~F.
\newblock Black-box optimization: Methods and applications.
\newblock In \emph{Black Box Optimization, Machine Learning, and No-Free Lunch
  Theorems}, pp.\  35--65. Springer, 2021.

\bibitem[Brockman et~al.(2016)Brockman, Cheung, Pettersson, Schneider,
  Schulman, Tang, and Zaremba]{brockman2016openai-gym}
Brockman, G., Cheung, V., Pettersson, L., Schneider, J., Schulman, J., Tang,
  J., and Zaremba, W.
\newblock Openai gym.
\newblock {arXiv}:1606.01540, 2016.

\bibitem[Cesa-Bianchi et~al.(2017)Cesa-Bianchi, Gentile, Lugosi, and
  Neu]{cesa2017boltzmann}
Cesa-Bianchi, N., Gentile, C., Lugosi, G., and Neu, G.
\newblock Boltzmann exploration done right.
\newblock In \emph{Advances in neural information processing systems},
  volume~30, 2017.

\bibitem[Chang \& Maxemchuk(1984)Chang and Maxemchuk]{chang1984reliable}
Chang, J.-M. and Maxemchuk, N.~F.
\newblock Reliable broadcast protocols.
\newblock \emph{ACM Transactions on Computer Systems (TOCS)}, 2\penalty0
  (3):\penalty0 251--273, 1984.

\bibitem[Chen et~al.(2021)Chen, Lu, Rajeswaran, Lee, Grover, Laskin, Abbeel,
  Srinivas, and Mordatch]{chen2021decision-transformer}
Chen, L., Lu, K., Rajeswaran, A., Lee, K., Grover, A., Laskin, M., Abbeel, P.,
  Srinivas, A., and Mordatch, I.
\newblock Decision transformer: Reinforcement learning via sequence modeling.
\newblock \emph{Advances in neural information processing systems},
  34:\penalty0 15084--15097, 2021.

\bibitem[Chen et~al.(2022)Chen, Zhang, Zhang, Wang, and
  Zhu]{chen2022byzantine-DRL}
Chen, Y., Zhang, X., Zhang, K., Wang, M., and Zhu, X.
\newblock Byzantine-robust online and offline distributed reinforcement
  learning.
\newblock \emph{arXiv preprint arXiv:2206.00165}, 2022.

\bibitem[Espeholt et~al.(2018)Espeholt, Soyer, Munos, Simonyan, Mnih, Ward,
  Doron, Firoiu, Harley, Dunning, et~al.]{espeholt2018impala-DRL-3}
Espeholt, L., Soyer, H., Munos, R., Simonyan, K., Mnih, V., Ward, T., Doron,
  Y., Firoiu, V., Harley, T., Dunning, I., et~al.
\newblock Impala: Scalable distributed deep-rl with importance weighted
  actor-learner architectures.
\newblock In \emph{International Conference on Machine Learning}, pp.\
  1407--1416. PMLR, 2018.

\bibitem[Fan et~al.(2021)Fan, Ma, Dai, Jing, Tan, and Low]{fan2021fault-FRL}
Fan, F.~X., Ma, Y., Dai, Z., Jing, W., Tan, C., and Low, B. K.~H.
\newblock Fault-tolerant federated reinforcement learning with theoretical
  guarantee.
\newblock In \emph{Advances in Neural Information Processing Systems}, 2021.

\bibitem[Fujita et~al.(2022)Fujita, Fujimura, Sun, Esaki, and
  Ochiai]{FedRL-building}
Fujita, K., Fujimura, S., Sun, Y., Esaki, H., and Ochiai, H.
\newblock Federated reinforcement learning for the building facilities.
\newblock In \emph{2022 IEEE International Conference on Omni-layer Intelligent
  Systems (COINS)}, pp.\  1--6, 2022.
\newblock \doi{10.1109/COINS54846.2022.9854959}.

\bibitem[Horgan et~al.(2018)Horgan, Quan, Budden, Barth-Maron, Hessel,
  Van~Hasselt, and Silver]{horgan2018-distributed-DRL-4}
Horgan, D., Quan, J., Budden, D., Barth-Maron, G., Hessel, M., Van~Hasselt, H.,
  and Silver, D.
\newblock Distributed prioritized experience replay.
\newblock \emph{arXiv preprint arXiv:1803.00933}, 2018.

\bibitem[Jin et~al.(2022)Jin, Peng, Yang, Wang, and
  Zhang]{jin2022federated-FRL-aistats}
Jin, H., Peng, Y., Yang, W., Wang, S., and Zhang, Z.
\newblock Federated reinforcement learning with environment heterogeneity.
\newblock In \emph{International Conference on Artificial Intelligence and
  Statistics}, pp.\  18--37. PMLR, 2022.

\bibitem[Johnson \& Zhang(2013)Johnson and Zhang]{johnson2013svrg1}
Johnson, R. and Zhang, T.
\newblock Accelerating stochastic gradient descent using predictive variance
  reduction.
\newblock In \emph{Advances in neural information processing systems}, pp.\
  315--323, 2013.

\bibitem[Kairouz et~al.(2021)Kairouz, McMahan, Avent, Bellet, Bennis, Bhagoji,
  Bonawitz, Charles, Cormode, Cummings, et~al.]{kairouz2021advances-FL-2}
Kairouz, P., McMahan, H.~B., Avent, B., Bellet, A., Bennis, M., Bhagoji, A.~N.,
  Bonawitz, K., Charles, Z., Cormode, G., Cummings, R., et~al.
\newblock Advances and open problems in federated learning.
\newblock \emph{Foundations and Trends{\textregistered} in Machine Learning},
  14\penalty0 (1--2):\penalty0 1--210, 2021.

\bibitem[Khodadadian et~al.(2022)Khodadadian, Sharma, Joshi, and
  Maguluri]{khodadadian2022federated-FRL-icml}
Khodadadian, S., Sharma, P., Joshi, G., and Maguluri, S.~T.
\newblock Federated reinforcement learning: Linear speedup under markovian
  sampling.
\newblock In \emph{International Conference on Machine Learning}, pp.\
  10997--11057. PMLR, 2022.

\bibitem[Kone{\v{c}}n{\`y} et~al.(2016)Kone{\v{c}}n{\`y}, McMahan, Ramage, and
  Richt{\'a}rik]{konevcny2016federated-FL-1}
Kone{\v{c}}n{\`y}, J., McMahan, H.~B., Ramage, D., and Richt{\'a}rik, P.
\newblock Federated optimization: Distributed machine learning for on-device
  intelligence.
\newblock \emph{arXiv preprint arXiv:1610.02527}, 2016.

\bibitem[Lattimore \& Szepesv{\'a}ri(2020)Lattimore and
  Szepesv{\'a}ri]{lattimore2020bandit-book}
Lattimore, T. and Szepesv{\'a}ri, C.
\newblock \emph{Bandit algorithms}.
\newblock Cambridge University Press, 2020.

\bibitem[Liang et~al.(2023)Liang, Liu, Chen, Liu, and
  Yang]{liang2019federated-FedRL-Car-YQ}
Liang, X., Liu, Y., Chen, T., Liu, M., and Yang, Q.
\newblock Federated transfer reinforcement learning for autonomous driving.
\newblock In \emph{Federated and Transfer Learning}, pp.\  357--371. Springer,
  2023.

\bibitem[Liu et~al.(2019)Liu, Wang, and Liu]{FedRL-1}
Liu, B., Wang, L., and Liu, M.
\newblock Lifelong federated reinforcement learning: A learning architecture
  for navigation in cloud robotic systems.
\newblock \emph{IEEE Robotics and Automation Letters}, 4\penalty0 (4):\penalty0
  4555--4562, 2019.
\newblock \doi{10.1109/LRA.2019.2931179}.

\bibitem[McMahan et~al.(2017)McMahan, Moore, Ramage, Hampson, and
  y~Arcas]{mcmahan2017communication-FL-0}
McMahan, B., Moore, E., Ramage, D., Hampson, S., and y~Arcas, B.~A.
\newblock Communication-efficient learning of deep networks from decentralized
  data.
\newblock In \emph{Artificial intelligence and statistics}, pp.\  1273--1282.
  PMLR, 2017.

\bibitem[Mnih et~al.(2013)Mnih, Kavukcuoglu, Silver, Graves, Antonoglou,
  Wierstra, and Riedmiller]{mnih2013playing-DQN}
Mnih, V., Kavukcuoglu, K., Silver, D., Graves, A., Antonoglou, I., Wierstra,
  D., and Riedmiller, M.
\newblock Playing atari with deep reinforcement learning.
\newblock {arXiv}:1312.5602, 2013.

\bibitem[Mnih et~al.(2016)Mnih, Badia, Mirza, Graves, Lillicrap, Harley,
  Silver, and Kavukcuoglu]{mnih2016asynchronous-DRL-2}
Mnih, V., Badia, A.~P., Mirza, M., Graves, A., Lillicrap, T., Harley, T.,
  Silver, D., and Kavukcuoglu, K.
\newblock Asynchronous methods for deep reinforcement learning.
\newblock In \emph{International conference on machine learning}, pp.\
  1928--1937. PMLR, 2016.

\bibitem[Nadiger et~al.(2019)Nadiger, Kumar, and Abdelhak]{FedRL-2}
Nadiger, C., Kumar, A., and Abdelhak, S.
\newblock Federated reinforcement learning for fast personalization.
\newblock In \emph{2019 IEEE Second International Conference on Artificial
  Intelligence and Knowledge Engineering (AIKE)}, pp.\  123--127, 2019.
\newblock \doi{10.1109/AIKE.2019.00031}.

\bibitem[Nair et~al.(2015)Nair, Srinivasan, Blackwell, Alcicek, Fearon,
  De~Maria, Panneershelvam, Suleyman, Beattie, Petersen,
  et~al.]{nair2015massively-DRL-1}
Nair, A., Srinivasan, P., Blackwell, S., Alcicek, C., Fearon, R., De~Maria, A.,
  Panneershelvam, V., Suleyman, M., Beattie, C., Petersen, S., et~al.
\newblock Massively parallel methods for deep reinforcement learning.
\newblock \emph{arXiv preprint arXiv:1507.04296}, 2015.

\bibitem[Papini et~al.(2018)Papini, Binaghi, Canonaco, Pirotta, and
  Restelli]{papini2018stochastic-svrpg}
Papini, M., Binaghi, D., Canonaco, G., Pirotta, M., and Restelli, M.
\newblock Stochastic variance-reduced policy gradient.
\newblock In \emph{International conference on machine learning}, pp.\
  4026--4035. PMLR, 2018.

\bibitem[Raffin et~al.(2021)Raffin, Hill, Gleave, Kanervisto, Ernestus, and
  Dormann]{raffin2021-stable-baselines3}
Raffin, A., Hill, A., Gleave, A., Kanervisto, A., Ernestus, M., and Dormann, N.
\newblock Stable-baselines3: Reliable reinforcement learning implementations.
\newblock \emph{Journal of Machine Learning Research}, 2021.

\bibitem[Rieke et~al.(2020)Rieke, Hancox, Li, Milletari, Roth, Albarqouni,
  Bakas, Galtier, Landman, Maier-Hein, et~al.]{rieke2020future-nature}
Rieke, N., Hancox, J., Li, W., Milletari, F., Roth, H.~R., Albarqouni, S.,
  Bakas, S., Galtier, M.~N., Landman, B.~A., Maier-Hein, K., et~al.
\newblock The future of digital health with federated learning.
\newblock \emph{NPJ digital medicine}, 3\penalty0 (1):\penalty0 1--7, 2020.

\bibitem[Sutton \& Barto(2018)Sutton and Barto]{sutton2018reinforcement-RL}
Sutton, R.~S. and Barto, A.~G.
\newblock \emph{Reinforcement learning: An introduction}.
\newblock MIT press, 2018.

\bibitem[Wang et~al.(2020)Wang, Wang, Li, Leung, and
  Taleb]{wang2020federated-FedRL-5}
Wang, X., Wang, C., Li, X., Leung, V.~C., and Taleb, T.
\newblock Federated deep reinforcement learning for internet of things with
  decentralized cooperative edge caching.
\newblock \emph{IEEE Internet of Things Journal}, 7\penalty0 (10):\penalty0
  9441--9455, 2020.

\bibitem[Watkins(1989)]{watkins1989learning-Q-learning}
Watkins, C. J. C.~H.
\newblock Learning from delayed rewards.
\newblock Ph{D thesis, University of Cambridge England}, 1989.

\bibitem[Xue et~al.(2021)Xue, Zhou, Xu, Wang, Xie, Ding, and
  Wen]{xue2021resource-FRL-clinical}
Xue, Z., Zhou, P., Xu, Z., Wang, X., Xie, Y., Ding, X., and Wen, S.
\newblock A resource-constrained and privacy-preserving edge-computing-enabled
  clinical decision system: A federated reinforcement learning approach.
\newblock \emph{IEEE Internet of Things Journal}, 8\penalty0 (11):\penalty0
  9122--9138, 2021.

\bibitem[Yahya et~al.(2017)Yahya, Li, Kalakrishnan, Chebotar, and
  Levine]{yahya2017collective-distributed-RL}
Yahya, A., Li, A., Kalakrishnan, M., Chebotar, Y., and Levine, S.
\newblock Collective robot reinforcement learning with distributed asynchronous
  guided policy search.
\newblock In \emph{2017 IEEE/RSJ International Conference on Intelligent Robots
  and Systems (IROS)}, pp.\  79--86. IEEE, 2017.

\bibitem[Yu et~al.(2020)Yu, Chen, Zhou, Gong, and Wu]{yu2020deep-FedRL-3}
Yu, S., Chen, X., Zhou, Z., Gong, X., and Wu, D.
\newblock When deep reinforcement learning meets federated learning:
  Intelligent multitimescale resource management for multiaccess edge computing
  in 5g ultradense network.
\newblock \emph{IEEE Internet of Things Journal}, 8\penalty0 (4):\penalty0
  2238--2251, 2020.

\bibitem[Zhuo et~al.(2019)Zhuo, Feng, Lin, Xu, and Yang]{zhuo2019federated}
Zhuo, H.~H., Feng, W., Lin, Y., Xu, Q., and Yang, Q.
\newblock Federated deep reinforcement learning.
\newblock {arXiv}:1901.08277, 2019.

\end{thebibliography}
\bibliographystyle{icml2023}

%%%%%%%%%%%%%%%%%%%%%%%%%%%%%%%%%%%%%%%%%%%%%%%%%%%%%%%%%%%%%%%%%%%%%%%%%%%%%%%
%%%%%%%%%%%%%%%%%%%%%%%%%%%%%%%%%%%%%%%%%%%%%%%%%%%%%%%%%%%%%%%%%%%%%%%%%%%%%%%
% APPENDIX
%%%%%%%%%%%%%%%%%%%%%%%%%%%%%%%%%%%%%%%%%%%%%%%%%%%%%%%%%%%%%%%%%%%%%%%%%%%%%%%
%%%%%%%%%%%%%%%%%%%%%%%%%%%%%%%%%%%%%%%%%%%%%%%%%%%%%%%%%%%%%%%%%%%%%%%%%%%%%%%
\newpage
\appendix
\onecolumn
\section{Algorithm Pseudocode}\label{appendix:sec:algorithm-pseudocode}
Due to space constraints, the pseudocode for the proposed FedHQL algorithm is deferred here and depicted in Algorithm~\ref{alg:1.a}\&\ref{alg:1.b}. The graphical illustration of FedHQL is also reproduced below (Fig.~\ref{fig:fedHQL-2}).

\begin{figure}[h]
    \centering
    % \noindent
    % \makebox[\textwidth]{\includegraphics[width=5.5in, height=1.2in] {./plots/EXP1.pdf}}
    \includegraphics[width=3in] {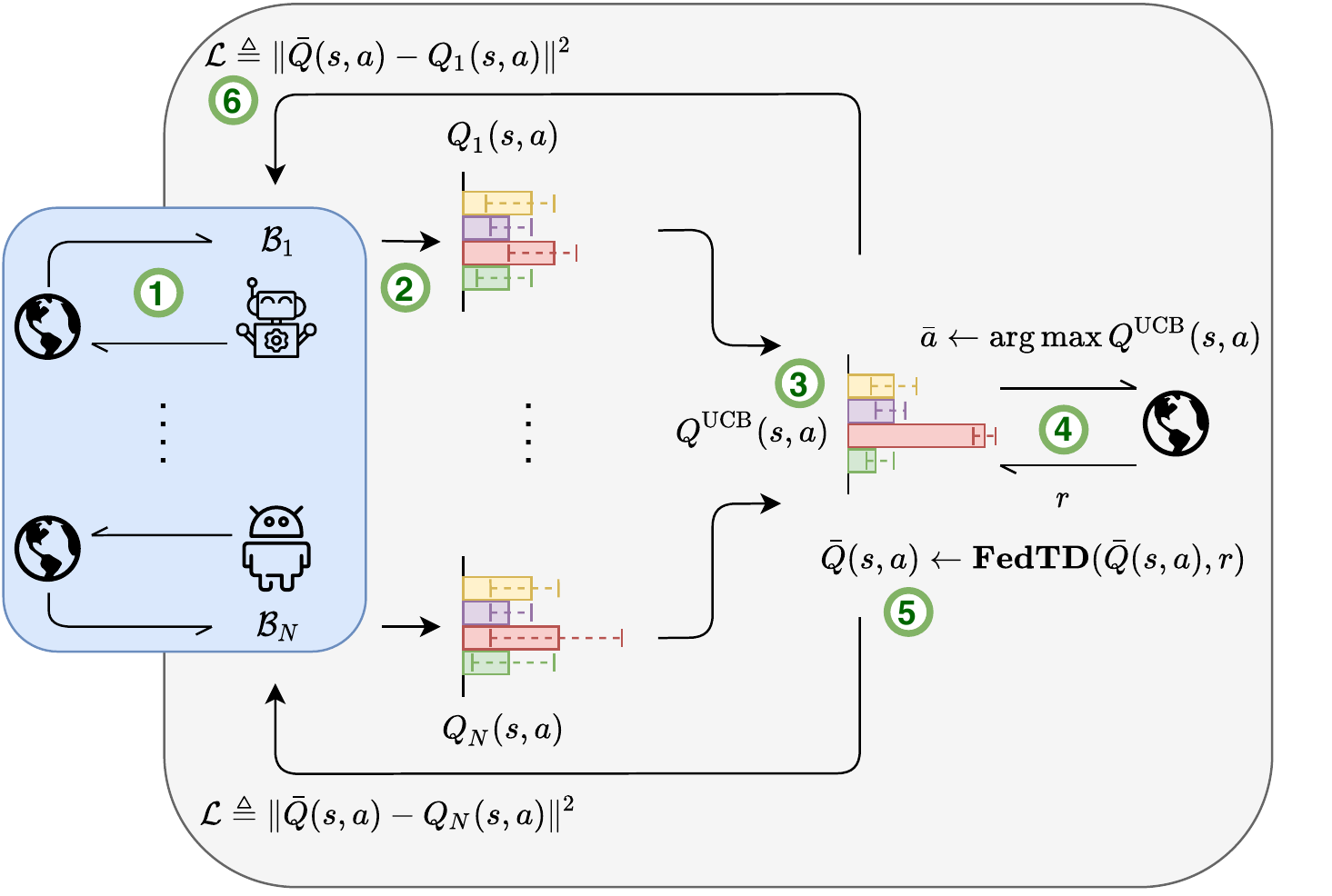}
    % \vspace{-8pt}
    \caption{Graphical illustration of FedHQL.}
    \label{fig:fedHQL-2}
\end{figure}
Of note, the above process can be implemented {in asynchronous batches} 
using the reliable broadcast protocol from distributed computing \cite{chang1984reliable}.
\begin{algorithm}[ht]
   \renewcommand\thealgorithm{1.a}
   \caption{FedHQL {(central server)}}
   \label{alg:1.a}
 \renewcommand{\algorithmicrequire}{\textbf{Input:}}
\renewcommand{\algorithmicensure}{\textbf{Output:}}
\begin{algorithmic}[1]
  \REQUIRE query state $s_0$, inter-agent exploration constant $\lambda$, {maximum federation time horizon $H_\text{fed}$, FedTD learning rate $\alpha_\text{s}$}
   \WHILE{a total budget of interactions is not exhausted}
\STATE // Begin self-learning process
\STATE Let each agent perform self-learning
\STATE // Begin federation process
    \STATE $t \gets 0$
        \WHILE{$s_{t}$ is not terminating \textbf{And} $t < H_\text{fed}$}
            \STATE {{Broadcast} $s_{t}$ to all agents; {Receive} $\{Q_n(s_t,a),\forall a\}_{n=1}^N$
            \STATE Compute $\bar{Q}(s_t,a), {Q}^{\text{std}}(s_t,a), Q^{\text{UCB}}(s_t,a), \forall a$ 
            % \myindent{1.5} 
            according to Eq.~\eqref{eq:Q-mean}, \eqref{eq:Q-std}, \eqref{eq:fed-AUCB} } %\tikzmark{right} \tikzmark{top}
            % \STATE Compute ${Q}^{\text{std}}(s_t,a),\forall a$ according to Eq.~\eqref{eq:Q-std}
            % \STATE Compute $Q^{\text{UCB}}(s_t,a),\forall a$ according to Eq.~\eqref{eq:fed-AUCB} %\tikzmark{bottom} 
            \STATE $\bar{a}_t \leftarrow \arg\max_{a} Q^{\text{UCB}}(s_t,a)$
            \STATE Sample $(s_t,\bar{a}_t,s_{t+1},r_t)$ by executing $\bar{a}_t$
            \STATE {{Broadcast} $s_{t+1}$ to all agents; {Receive} $\{Q_n(s_{t+1}, \forall a)\}_{n=1}^N$}
            \STATE Perform $ \textbf{FedTD}(\bar{Q}(s_t,\bar{a}_t), r_t, \alpha_\text{s})$ according to Eq.~\eqref{eq:fed-TD}
            \STATE {{Broadcast} $\bar{Q}(s_t,\bar{a}_t)$ to all agents; {Wait} for agents to perform individual improvement
            \STATE $t \gets t + 1$} %{\color{green}\textbf{\textcircled{{\scriptsize 1}}}, \textbf{\textcircled{{\scriptsize 2}}}}
        \ENDWHILE
    \ENDWHILE
\end{algorithmic}
% \AddNote{top}{bottom}{right}{\textbf{\textcircled{{\scriptsize 3}}}}
\end{algorithm}

\begin{algorithm}[!ht]
\renewcommand\thealgorithm{1.b}
   \caption{FedHQL (agent $n$)}
   \label{alg:1.b}
\renewcommand{\algorithmicrequire}{\textbf{Input:}}
\renewcommand{\algorithmicensure}{\textbf{Output:}}
\begin{algorithmic}[1]
    % \STATE {Input:} sample size
   \REQUIRE arbitrary network $Q_n$ parameterized by $\theta_n$, intra-agent exploration coefficient $\epsilon_n$, {individual learning rate $\alpha_n$, number of gradient steps $\kappa$}
{ 
    \IF{Self-learning signal received from server}
   \STATE Perform self-learning: $Q_n \leftarrow$ Self-learning($Q_n$, $\epsilon_n$, $\alpha_n$)
   \ENDIF
   \IF{$\bar{Q}(s_t,\bar{a}_t)$ received from server}
        \STATE Update $Q_n$ according to Eq.~\eqref{eq:individual-improvement-loss} \& \eqref{eq:individual-improvement-update} for  $\kappa$ steps
   \ENDIF
   \IF{$s_t$ received from server}
        \STATE Answer query $Q_n$ at received $s_t$ and send it to central server
   \ENDIF
   }
\end{algorithmic}
\end{algorithm}
\section{Experiments with More Agents}\label{appendix:exp-more-agents}
We repeat the same sets of experiments with $N=10$ agents, whose configurations are depicted in Table~\ref{tab:configuration-2}, to further verify the efficacy of FedHQL in improving the system welfare and its effectiveness in improving the sample efficiency of individual participating agents. 
\begin{table}[h]
\centering
  \caption{Configurations of $N=10$ agents}
  \label{tab:configuration-2}
  \resizebox{0.45\textwidth}{!}{%
  \begin{tabular}{cccc}\toprule
    \textit{Agent} & \textit{Network} &  \textit{Learning rates} &  \textit{Intra-exploration coefficient} \\ \midrule
    1 & 64x64 (ReLU) & 0.01 & 0.01 \\
    2 & 128x128 (ReLU) & 0.1 & 0.1 \\
    3 & 32x32 (Tanh) & 0.01 & 0.05 \\
    4 & 16x16 (ReLU) & 0.01 & 0.01 \\
    5 & 8x8x8 (ReLU) & 0.01 & 0.01 \\ 
    6 & 64x64 (Tanh) & 0.02 & 0.01 \\
    7 & 128x128 (ReLU) & 0.02 & 0.1 \\
    8 & 32x32 (ReLU) & 0.02 & 0.05 \\
    9 & 16x16 (Tanh) & 0.05 & 0.01 \\
    10 & 8x8x8 (Tanh) & 0.05 & 0.01 \\ \bottomrule
  \end{tabular}
  }
\end{table}

\subsection{Efficacy of FedHQL ($N=10$ agents) in Improving System
Welfare}
Similar conclusion from Sec.~\ref{subsec:exp-object-1} and Sec.~\ref{subsec:exp-object-2} can also be drawn from both the two tasks.
For example, 
in the CartPole task,
an agent needs to sample at least $20\%$ of its budget (i.e., total $0.4\text{m} = 2\times 10^6 \times0.2$ interactions) on average to receive an average return of 300 while an agent in FedHQL ($\lambda=1,3$) can achieve a performance close to 400 within the same consumption of $20\%$ of its budget.
Also in the LunarLander task,
an agent is expected to consume at least $20\%$ of its budget (i.e., total $0.8\text{m} = 4\times 10^6 \times0.2$ interactions) on average to receive positive returns while an agent in FedHQL ($\lambda=3$) can achieve a performance close to 100 within the same consumption of $20\%$ of its budget.
These performance improvements over the $N=10$ heterogeneous and black-box agents further verify the empirical performance of FedHQL in improving the sample efficiency of RL agents from the system perspective.
\begin{figure}[h]
    \centering
    \includegraphics[width=3.3in]{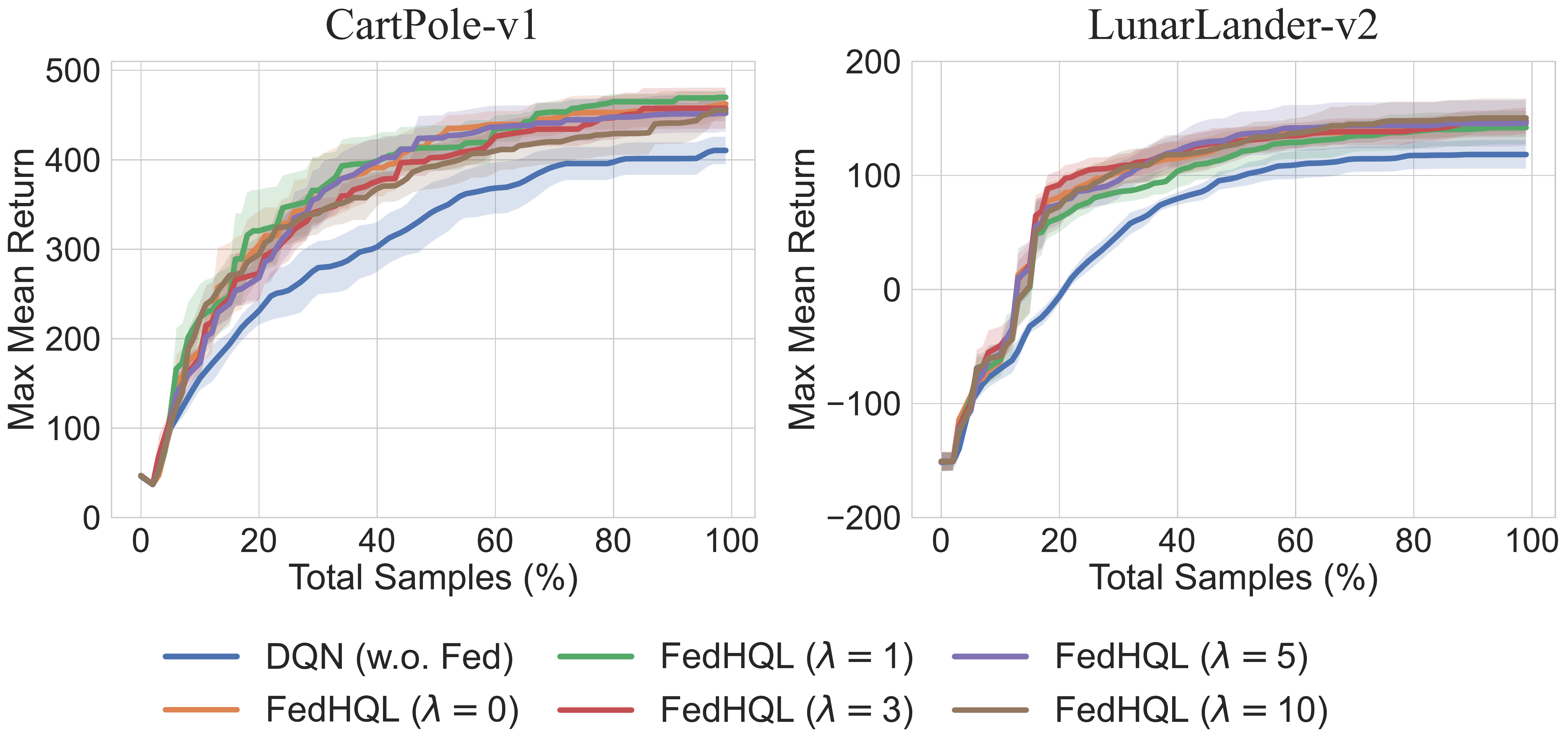}
    \caption{Learning curves of FedHQL ($N=10$ agents) with different inter-agent exploration coefficients (FedHQL ($\lambda\!=\!0,1,3,5,10$) against independent self-learning of agents (DQN (w.o. Fed)) on the CartPole (\emph{left}) and LunarLander (\emph{right}) environments. }
    \label{fig:exp2}
\end{figure}

\subsection{Effectiveness of FedHQL ($N=10$ agents) in Improving
Individual Agents}

\begin{figure*}
    \centering
  \includegraphics[width=\textwidth] {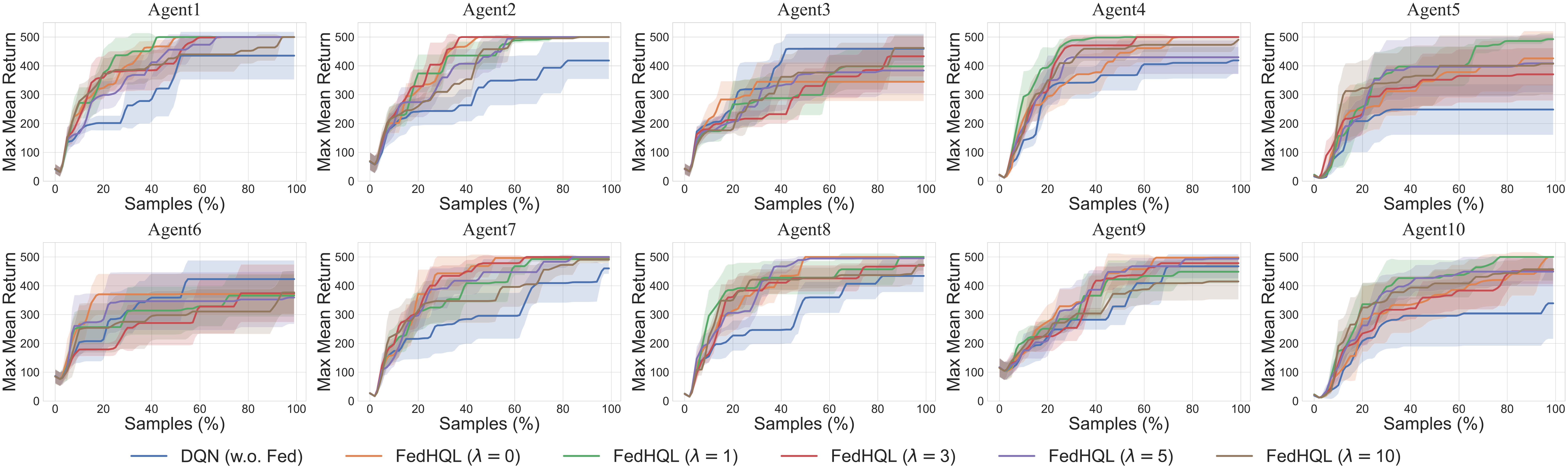}
    \caption{Individual learning curves of each agent running FedHQL ($N=10$ agents) with different inter-agent exploration coefficients (FedHQL ($\lambda\!=\!0,1,3,5,10$) against independent training (DQN (w.o. Fed)) on the CartPole environment.}
    \label{fig:exp2-cartpole}
\end{figure*}
Fig.~\ref{fig:exp2-cartpole} and Fig.~\ref{fig:exp2-lunarlander} show the results used to verify the effectiveness of FedHQL in boosting the sample efficiency of each participating agent with high probability, i.e., the second objective of FedRL-HALE 
(Sec.~\ref{subsec:problem-formulation}), for CartPole and LunarLander, respectively. 
Interestingly, FedHQL ($\lambda=0$) renders the performance of Agent3 slightly worse than their corresponding independent self-learning (DQN (w.o. Fed)).
This is because Agent5 and Agent 10 (whose policy networks are poorly parameterized) fail to learn the CartPole task within the given budget, causing inaccurate knowledge (estimation of the action values) of Agent5 and Agent 10 to be sent to the server.
As discussed in Sec.~\ref{sec:fed-ucb-q}, FedUCB with $\lambda = 0$ encourages the server to strongly exploit the group knowledge $\bar{Q}$ which is inaccurate due to the misleading knowledge contributed by Agent5 and Agent 10.
As a result, the performances of Agent3 in FedHQL ($\lambda = 0$) deteriorate slightly in the early training stage.
Moreover, it is worth noting that FedHQL ($\lambda=0,1,3,5,10$) significantly improves the performance of Agent5 and Agent10 from being completely unlearnable to 
solving the task with nearly maximum return, which further verifies the effectiveness of FedHQL in improving individual agents with high probability.

% The plots suggest that, for both tasks, FedHQL is able to boost the sample efficiency for most agents and can even improve the performance of some agents (i.e., Agent5 and Agent10) from completely unlearnable to solving the task with nearly maximum return.

\begin{figure*}
    \centering
    \includegraphics[width=\textwidth] {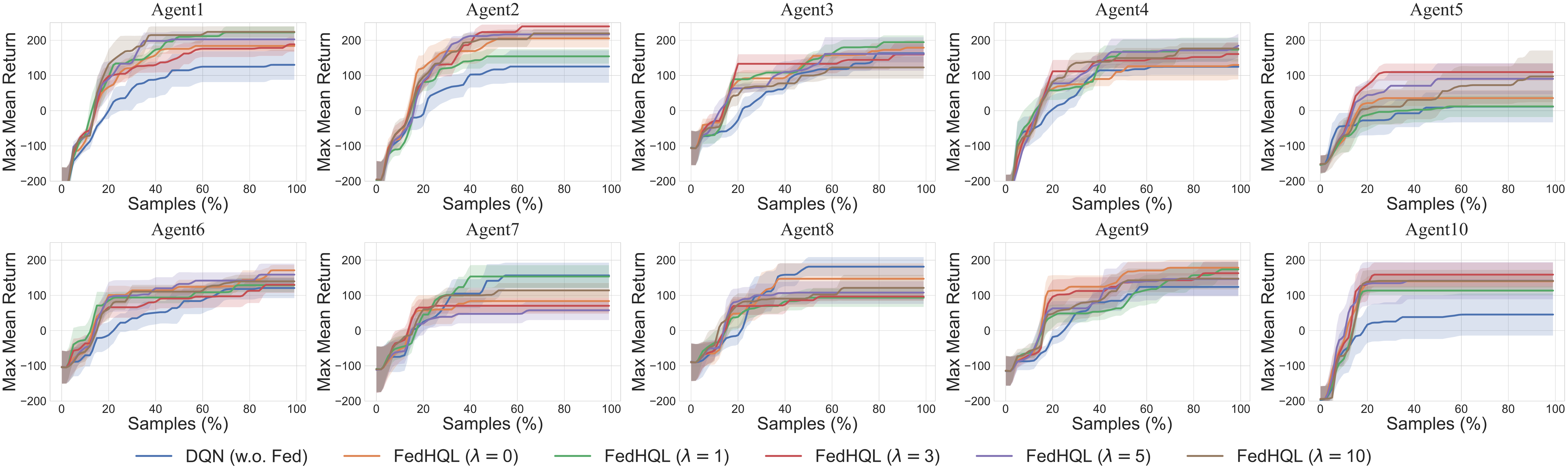}
    \caption{Individual learning curves of each agent running FedHQL ($N=10$ agents) with different inter-agent exploration coefficients (FedHQL ($\lambda\!=\!0,1,3,5,10$) against independent training (DQN (w.o. Fed)) on the LunarLander environment.
    }
    \label{fig:exp2-lunarlander}
\end{figure*}

\end{document}